%% file: paper.tex
\documentclass{llncs}

\usepackage[usenames,dvipsnames,svgnames,table]{xcolor}% http://ctan.org/pkg/xcolor
\usepackage{latexsym}
\usepackage[ruled,vlined]{algorithm2e}
\usepackage{tikz}
\usepackage{xspace}
\usepackage{amssymb}
\usepackage{subfig}%
\usepackage{multirow}
\usepackage{rotating}
\usepackage{accents}
\usepackage{pifont}
\usepackage{url}

\usepackage{tikz}
\usepackage{fp}

\usepackage{enumitem}

\usetikzlibrary{arrows,shadows,fit,calc,positioning,decorations.pathreplacing,matrix,shapes,petri,topaths,fadings,mindmap,backgrounds}

\def\onmrk{\ding{53}}
\def\rqmrk{\ding{54}}

\newcommand{\constraint}[1]{\mbox{\sc #1}\xspace}
\newcommand{\algo}[1]{\mbox{\textbf{#1}}\xspace}
\newcommand{\instance}[1]{\mbox{\texttt{#1}}\xspace}

\newcommand{\bfr}{\constraint{BufferedResource}}
\newcommand{\switch}{\constraint{Switch}}
\newcommand{\switchp}{\constraint{Switch$^{+}$}}

\newcommand{\alldiff}{\constraint{AllDifferent}}

\newcommand{\texmath}[1]{\ensuremath{#1}\xspace}

\def\instsymb{I}
\def\aninstance{\texmath{\instsymb}}
\def\anotherinstance{\texmath{\instsymb'}}

\def\setsymbol{\varactivesymb}
\def\solactsymbol{\sigma}
\def\solallsymbol{\tau}
\def\countvarc{N}
\def\countvars{M}
\def\countvart{K}

\newcommand{\dom}[1]{\texmath{D(#1)}}
\newcommand{\domcol}[1]{\texmath{\widehat{D}(#1)}}

\newcommand{\setvar}[1]{\texmath{\setsymbol_{#1}}}

\def\abufseq{\texmath{\solactsymbol}}
\def\anotherbufseq{\texmath{\solactsymbol'}}
\def\analloc{\texmath{\solallsymbol}}
\newcommand{\bufseq}[1]{\texmath{\solactsymbol(#1)}}
\newcommand{\otherbufseq}[1]{\texmath{\solactsymbol'(#1)}}
\newcommand{\alloc}[1]{\texmath{\solallsymbol(#1)}}

\def\testsymb{t}
\def\ntestsymb{n}
\def\itestsymb{k}
\def\ntest{\texmath{\ntestsymb}}
\def\itest{\texmath{\itestsymb}}

\def\varallocsymb{X}
\newcommand{\varalloc}[1]{\texmath{\varallocsymb_{#1}}}

\def\nslotsymb{h}
\def\islotsymb{i}
\def\nslot{\texmath{\nslotsymb}}
\def\islot{\texmath{\islotsymb}}

\def\varnumconf{\texmath{\countvarc}}
\newcommand{\varnumswitch}[1]{\texmath{\countvars_{#1}}}

\def\unitsymb{e}
\def\nunitsymb{m}
\def\iunitsymb{j}
\def\nunit{\texmath{\nunitsymb}}
\def\iunit{\texmath{\iunitsymb}}

\def\equipsymb{V}
\def\items{\texmath{U}}
\def\someitems{\texmath{\equipsymb}}
\def\anitem{\texmath{\iunitsymb}}
\def\anotheritem{\texmath{\iunitsymb'}}
\def\afirstitem{\texmath{\iunitsymb_1}}
\def\aseconditem{\texmath{\iunitsymb_2}}

\def\varactivesymb{Y}
\newcommand{\varon}[2]{\texmath{\varactivesymb_{#2}^{#1}}}
\newcommand{\varntested}[1]{\texmath{\countvart_{#1}}}

\def\varactualsymb{Z}
\newcommand{\hason}[2]{\texmath{\varactualsymb_{#2}^{#1}}}

\def\conssymb{c}
\def\nconssymb{p} % Change 
\def\iconssymb{\ell}
\def\ncons{\texmath{\nconssymb}}
\def\icons{\texmath{\iconssymb}}

\def\cardsymb{\kappa}
\def\bufsize{\texmath{\cardsymb}}
\def\tsizesymb{\Delta}

\newcommand{\thermalsize}[1]{\texmath{\tsizesymb_{#1}}}
\newcommand{\constequip}[1]{\texmath{\conssymb}_{#1}}
\newcommand{\thermalbound}[1]{\texmath{\cardsymb_{#1}}}

\def\numswitch{\texmath{\alpha}}
\def\numstretch{\texmath{\beta}}
\def\numopt{\texmath{\gamma}}
\newcommand{\equipof}[1]{\texmath{\testsymb}_{#1}}

\newcommand\thickbar[1]{\accentset{\rule{.5em}{.8pt}}{#1}}

\newcommand{\maxdom}[2]{\texmath{\thickbar{#1}_{#2}}}
\DeclareRobustCommand{\mindom}[2]{\texmath{\underaccent{\thickbar{}}{#1}_{#2}}}

\newcommand{\ubS}[1]{\maxdom{\setsymbol}{#1}}
\newcommand{\lbS}[1]{\mindom{\setsymbol}{#1}}

\newcommand{\ubON}[2]{\maxdom{\varactivesymb}{#1}^{#2}}
\newcommand{\lbON}[2]{\mindom{\varactivesymb}{#1}^{#2}}

\newcommand\mystretch[3]{\texmath{\langle#1,#2,#3\rangle}}

\newcommand\nextTimeInLowerBound[2]{\texmath{next^{#2}_{\in}({#1})}}
\newcommand\nextTimeNotInUpperBound[2]{\texmath{next^{#2}_{\not\in}({#1})}}
\def\betterthan{\texmath{\prec}}

\def\findsupport{\algo{FindSupport}}

\newcommand{\maxcard}[1]{\texmath{\thickbar{\cardsymb}_{#1}}}
\DeclareRobustCommand{\mincard}[1]{\texmath{\underaccent{\thickbar{}}{\cardsymb}_{#1}}}
\newcommand{\maxcardinality}[1]{\maxcard{#1}}
\newcommand{\mincardinality}[1]{\mincard{#1}}

\newcommand{\ceil}[1]{\left\lceil{#1}\right\rceil}

\newtheorem{observation}{Observation}

\usepackage{xcolor}
\def\mjo#1{\textcolor{black}{#1}}  % Ajout MJO
\def\rque#1{}  % Questions
\def\toto#1{\textcolor{black}{#1}}  % Ajout MJO
\def\mj#1{\textcolor{black}{#1}}  % Ajout MJO

\title{Constraint Programming for Planning Test Campaigns of Communications Satellites}

\author{
Emmanuel Hebrard$^1$ \and
Marie-Jos\'e Huguet$^1$ \and
Daniel Veysseire$^1$ \and
Ludivine Boche Sauvan$^{1,2}$ \and 
Bertrand Cabon$^2$
}

\institute{
$^1$LAAS-CNRS, Universit\'e de Toulouse, CNRS, INSA, Toulouse, France \\
$^2$Airbus Defence \& Space, Toulouse, France
}

\begin{document}

\maketitle

\begin{abstract}
	The payload of communications satellites must go through a series of tests to assert their ability to survive in space.
	Each test involves some equipment of the payload to be active, which has an impact on the temperature of the payload. 
	Sequencing these tests in a way that ensures the thermal stability of the payload and minimizes the overall duration of the test campaign is a very important objective for satellite manufacturers.
	 
	The problem can be decomposed in two sub-problems corresponding to two objectives:
	First, the number of distinct configurations necessary to run the tests must be minimized. % to avoid time-costly transitions.
	%in order to reduce the total duration of transitions. 
	This can be modeled as packing the tests into configurations, and we introduce a set of implied constraints to improve the lower bound of the model.
	
	Second, tests must be sequenced so that the number of times an equipment unit has to be switched on or off is minimized. We model this aspect using the constraint \switch, where %active equipment 
	a buffer with \mj{limited} capacity represents the currently active equipment units,
	%currently active equipment is viewed as a buffer with \mj{limited} capacity 
	%and switches of buffered \mj{equipment units} are to be minimized, 
	% which number of activations must be minimized,
	and we introduce an improvement of the propagation algorithm for this constraint.
	
	\mj{We then introduce a search strategy in which we sequentially solve the sub-problems (packing and sequencing). %packing and then the sequencing sub-problem. % (packing and sequencing). 
	Experiments conducted on real and random instances show the respective interest of our contributions.}
\end{abstract}

\section{Introduction}
\label{sec:int}

% These Caroline
% Papier Mosim

The payload of a communications satellite is the on-board equipment that is actually relevant to the mission:
receiving, amplifying and returning the signal.
 %it refers, %for
%. In a telecommunication satellite, the payload gathers the various equipment which are crossed by a RF signal. The payload 
%In a 
%In the  communications satellites, it refers to the equipment that
%receives, amplifies and returns the signal. %, but also drives it between the different elements. 
%It must be therefore be tested under a simulated space environment.
A set of tests are necessary to certify that the payload will correctly perform its mission. A test is characterized by a set of active equipment units and several thermal constraints limit the number of equipment units that can be made active simultaneously. %, from a equipment subset. 
The duration of the tests themselves is incompressible. However, activating some equipment takes time as each equipment unit much reach a given temperature and the temperature of the entire payload must be stabilized before tests can resume. Therefore, the total transition time depends on the order in which tests are sequenced. The goal is to sequence all the tests so that the overall duration is minimized.

Two main approaches have been previously considered. In the first approach~\cite{maillet2011constraint}, tests requiring the same subset of active equipment are packed together in the same payload configuration. 
The transition time between tests run in the same configuration is null since no equipment activation is required.
Between two configurations, however, some equipment units must be activated or deactivated and the heat of the payload must be stabilized before the next configuration, which entails important transition times. 
The objective of this approach is to minimize the number of configurations necessary for running all the tests.
Moreover, a secondary objective is to minimize the overall number of activations and deactivations since the time required to
stabilize the temperature of the payload is correlated, though not linearly, to the number of simultaneous activations. 
\mj{In practice, the transition time between two configurations is considered as a constant value sufficient to stabilize the temperature in the payload.}
% and then reducing the number of shifts for the equipment units.
Maillet et al.~\cite{maillet2011constraint} first proposed a constraint programming approach to address this problem, using backjumping and dedicated heuristics.
%, and adaptive mechanisms based on failures and restarts~\cite{maillet2011constraint}.
%In~\cite{maillet2011constraint}, to address this problem, a dedicated method based on chronological backtracking and several improvements (backjumping, dynamic heuristics, adaptive mechanisms based on failures and restart) was developed. 
%The proposed method produces a solution for tests planning involving several thousand\rque{hundred?} of tests for real payloads in less than five minutes.

It is in principle possible to activate an equipment unit whilst tests on other units are being run.  
The second approach, presented in~\cite{bochesauvan:hal-01166683}, relies on this idea. An equipment activation is viewed as a task with a given duration. 
%Suppose that an equipment unit 
If an equipment unit is not tested during a period equal to this duration and if the thermal constraints allow it, then it can be activated 
at the beginning of this period,
and becomes available for other tests at the end.
%during that period, whithout stopping the test sequence. 
It may be possible to do so for several or even all equipment activations, thus effectively masking the transition times.
% % and if the thermal constraint allowsIf, prior to a test of a series of tests requiring a
% Focusing on the transition step that may be time consuming, the second approach~\cite{bochesauvan:hal-01166683} aims to sequence tests one by one by activating (or dis-activating) equipment for a next test during the operating of previous ones.
% The objective is to mask all the shifts of equipment during the operating of tests, in other words to minimize the number of unmasked shifts.
A local search method (simulated annealing) was proposed in~\cite{bochesauvan:hal-01166683} for this second approach, that is, with ``online'' activation allowed. 
%This method is based on two neighborhoods to swap two tests in the sequence or to shift the status of an equipment. 
%The experimental evaluation conducted in this paper showed 
It was shown 
that this approach can reduce the overall duration of the test campaign.
%the time spend for shifting equipments comparatively to the first approach, but in a large amount of cpu time. 
However, the total number of activations can be higher in some instances. %Even if it is masked by other tests in, shifting an equipment requires some attention on the thermal conditions of the payload.

Finally, since the focus in \cite{maillet2011constraint} was on selecting a subset of tests to be run rather than sequencing them, a straightforward  
improvement of \mj{the first} approach was proposed in~\cite{bochesauvan:hal-01166683}:
Once every test is allocated to a configuration and the number of configurations is minimized, 
the problem of minimizing the number of activations between consecutive configurations can be seen as a Traveling Salesman Problem (TSP). As the number of configurations is typically small, this problem can be solved optimally even with relatively basic TSP methods.
\medskip

%However, this 
\mj{The} second approach is difficult to implement in practice as activations and deactivations can happen continuously, making the detailed thermal analysis of the process a difficult task, whereas thermal engineers only need to worry about transitions between configurations in the first approach. 
In this work, we present a study for Airbus, which currently only implements the first approach.  We propose an improved constraint programming approach to solve the problem.
%
% Therefore, only the first approach is currently implemented at Airbus and
% %we focus on optimizing it in this paper.
% %In this paper, we consider that online activations are not allowed and
% we propose an improved constraint programming approach \mj{to solve it}.
 
Since thermal constraints limit the number of tests that can be run in the same configuration, this problem has a packing component, where tests have to be allocated to a minimum \mj{number} of configurations. %and the number of configuration must be minimized. 
We introduce a set of implied constraints to improve the lower bound on this objective.

% \mjo{In this paper, we consider the first approach for the test sequencing. And, we aim to propose efficient constraint programming models for minimizing the number of configurations and the number of shifts by the way of global constraints suitable for these problems.}
%
% \mjo{Note that in the real application, tests are characterized by two types of equipment units. The first type corresponds to equipment that must be activated or not and the second one corresponds to switches that must be turned in the right position (to let the signal cross the payload). But, this second type of equipment does not really constrained the test sequencing and it is not taking into account in this paper.}

Moreover, this packing component is 
%This problem can be seen as a type of packing problem where  
intertwined with a sequencing component where configurations must be ordered so that the number of equipment activations is minimized.
%
% %First, the tests must be packed in as few configurations as possible.
% %This problem can be seen as a type of packing problem as thermal constraint limit the capacity of the bins/configurations.
% For the subproblem of packing the tests in as few configurations as possible,
% we introduce a set of implied constraints to improve the lower bound. % of the model for the number of configurations. % in Section~\ref{sec:pac}.
% %The secondary objective  the configurations must be sequenced to minimize the overall number of activations.
%
% Then, 
We observe that this secondary objective %of minimizing the overall number of activations 
can be easily modeled using the constraint \switch~\cite{DBLP:conf/cpaior/BessiereHMQW14}.
This constraint models a resource operating as a buffer of limited capacity and whose content can be changed, however at a cost.
%
%
% which models, \mj{for the studied problem}, a \mj{buffer}
% %resource
% containing \mj{equipment}
% %items
% %(here equipment units)
% required by some \mj{tests,}
% %tasks
% %(here tests)
% %with limited capacity (due to thermal constraints)
% %and counts the number switches.
% \mj{ and subject to both a limited capacity and a limitation on the number of equipment unit's switches.}
% %(here an equipment unit's activation).
Tests require equipment to be active and thermal constraints limit the number of simultaneously active equipment units, which can thus be seen as buffered. Moreover,
the number of activations, which correspond to switches in the buffer, should be minimized.
%. Therefore, we can view the set of equipment units that are \emph{active} to be \emph{buffered}. 
% which counts the number of switches in a buffer of items with a capacity. Indeed, thermal constraints limit the number of equipment units that are activated at a given time, 
%the capacity of the buffer is limited thermal constraints 
%
% (e.g., thermal constraints). 
%In Section~\ref{sec:seq} 
%
We introduce a simple improvement of the propagation algorithm for this constraint in the very common case where we have prior knowledge about the items that must be eventually buffered. 

Next, we introduce a search strategy in which we sequentially solve the two sub-problems.
We solve first the packing problem with a dedicated branching heuristic to find upper and lower bounds on this objective quickly. 
Then we solve a much simplified sequencing problem since tests are already allocated to configurations. This approach is not complete, and hence we must solve the overall problem in order to find optimal solutions. However this is made significantly easier thanks to the upper and lower bounds found in the previous phases.

%Finally, we solve

%This method sometimes allow to prove optimality with respect to the first objective, however,

Finally, we experimentally evaluate the different contributions and assess the benefit of our method with respect to the current approach in use at Airbus.

\section{Formal Background}
\label{sec:pre}

%\subsection{Notations}

A constraint satisfaction problem (CSP) consists of a set of variables,
where each variable $\varalloc{\itest}$ has a 
finite domain of values $\dom{\varalloc{\itest}}$, and a set of constraints
specifying allowed combinations of values for subsets of
variables. 
A solution of a CSP is an assignment of
values to the variables satisfying the constraints. 

We consider both
\emph{integer} and \emph{set} variables.
A set variable \varon{}{\islot} is represented by
its lower bound $\lbS{\islot}$ which contains the required elements % of the set
and an upper bound $\ubS{\islot}$ which contains the %definite and
possible elements. For a finite universe $\items \subset \mathbb{N}$, we identify a set variable \varon{}{\islot} with the set of Boolean variables $\{\varon{\iunit}{\islot} \mid \iunit \in \items\}$. 
%where $\items$ is some set of integers. 
The predicates $\varon{\iunit}{\islot} = 1$ and $\iunit \in \varon{}{\islot}$ are equivalent, as are $|\varon{}{\islot}| = \cardsymb$ and $\sum_{\iunit \in \unitsymb} \varon{\iunit}{\islot} = \cardsymb$.

For two integers $a \leq b$, we denote  $[a,b]$ the set of consecutive integers $\{a,\ldots,b\}$, and use the shortcut notation $[b]$ for $[1,b]$.

We shall often denote a sequence of variables or constants $(c_{1},\ldots,c_{\ntest})$ by $c$ (where the length $\ntest$ is either recalled or clear from the context).

\subsection{The constraints \switch and \bfr}

%Thermal constraints define subset of equipment units  

The constraints \switch and \bfr~\cite{DBLP:conf/cpaior/BessiereHMQW14} were introduced to model a type of resource corresponding to a \emph{buffer} which must contain the \emph{items} required by some tasks when they are being processed. Such resources are limited in two ways: first, 
the buffer can only hold a limited number of items,
%only a limited number of items can be simultaneously in the buffer, 
and second, there is an upper bound on the number of item \emph{switches} along the sequence.

\mj{In our context,} these constraints are useful to model thermal constraints together with the objective to minimize the number of activations (switches) of equipment units (items to be buffered). 
%We need to count the number of activations of equipment units, which is neatly captured by the \switch constraint~\cite{DBLP:conf/cpaior/BessiereHMQW14}. Each thermal constraint can be seen as specifying the capacity of a buffer of active equipment units, and when a test is ran, all the units it requires must be in the buffer. 
The constraint \switch involves a sequence of set variables $\varon{}{} = (\varon{}{1},\ldots,\varon{}{\ntest})$ and an integer variable $\varnumswitch{}$. The set variable $\varon{}{\islot}$ represents the content of the buffer at position (or time) $\islot$, and the variable $\varnumswitch{}$ represents the total number of items that are removed from the buffer to make room for new items: $\sum_{\islot=1}^{\ntest-1} |\varon{}{\islot} \setminus \varon{}{\islot+1}|$.
Moreover, the buffer has a minimum and maximum capacity (lower and upper bound on $|\varon{}{\islot}|$) which are allowed to be different at every position $\islot$. Let $\setvar{}$ be a sequence of set variables, and $\mincardinality{},\maxcardinality{}$ be two sequences of constants of same size $\ntest$.

\begin{definition}[\switch]\label{def:switch}
  %For each $\islot \in [\nslot]$, let $\mincard{\islot} < \maxcard{\islot}$ be two integers.
  \begin{eqnarray*}
  % \switch([\setvar{1},\ldots,\setvar{\nslot}],[\mincardinality{1}, \ldots, \mincardinality{\nslot}], [\maxcardinality{1}, \ldots, \maxcardinality{\nslot}], \varnumswitch{}) \iff \\
	\switch(\setvar{},\mincardinality{},\maxcardinality{},\varnumswitch{}) \iff
  \forall \islot \in [\ntest],~ \mincardinality{\islot} \leq |\setvar{\islot}| \leq \maxcardinality{\islot} ~\wedge~ \sum_{1 \leq \islot < \ntest}|\setvar{\islot+1} \setminus \setvar{\islot}| \leq \varnumswitch{}
	\end{eqnarray*}
\end{definition}

Often, %and this is the case in our test planning problem, 
we know beforehand which tasks are to be performed and which items are required by each task. In this case, one can use the \bfr constraint which involves a sequence of integer variables $\varalloc{} = (\varalloc{1},\ldots,\varalloc{\ntest})$ representing a permutation of $\ntest$ tasks, and 
a sequence $\equipof{} = (\equipof{1},\ldots,\equipof{\ntest})$ of sets of integers
%$\nslot$ sets $\{\equipof{\islot} \mid \islot \in [\nslot]\}$ 
standing for the items required by each task.
Achieving arc consistency on this constraint is NP-hard, and there is no dedicated propagation algorithm for this constraint, besides the obvious decomposition using \switch, \alldiff~\cite{regin1} and some \constraint{Element}~\cite{Hentenryck88} constraints. 

 %, however, we show in this paper that one can use the prior information about task requirement to improve the propagation algorithm for \switch in this case.
%Let $\varalloc{}$ and $\equipof{}$ denote the sequences $(\varalloc{1},\ldots,\varalloc{\nslot})$ and $(\equipof{1},\ldots,\equipof{\ntest})$, respectively.

\begin{definition}[\bfr]\label{def:bfr}
	\begin{eqnarray*}
  \bfr(\varalloc{},\setvar{},\mincardinality{},\maxcardinality{},\equipof{},\varnumswitch{}) & \iff \\
	\switch(\setvar{},\mincardinality{},\maxcardinality{}, \varnumswitch{}) & ~\wedge~ \\
	\forall \islot < j \in [1,\nslot],~ \varalloc{\islot} \neq \varalloc{j} & ~\wedge~ \\
	\forall \islot,~ \equipof{\islot} \subseteq \setvar{\varalloc{\islot}} &
\end{eqnarray*}
\end{definition}

In our test planning problem, the buffers have equal upper and lower bounds. We shall therefore use a single integer parameter to denote the sequences $\mincardinality{}$ and $\maxcardinality{}$ in the remainder of the paper.

\section{Test Planning}
\label{sec:pro}

\subsection{Data and Constraints}

A test campaign involves $\ntest$ tests and $\nunit$ equipment units. 
Every test $\itest \in [\ntest]$ involves a subset $\equipof{\itest} \subseteq [\nunit]$ of equipment units to be active.\footnote{Throughout the paper, an equipment unit is said to be ``active'' if it is switched on and ``inactive'' otherwise.}
A payload configuration (or simply configuration) is defined by a partition of the equipment into active and inactive units.
A test $\itest$ can occur in a configuration if the set of active equipment units in that configuration is a superset of $\equipof{\itest}$.
%The first objective is to allocate every test to a configuration whilst minimizing the number of configurations, in order to minimize the number of times the configuration of the satellite needs to be changed.

However, one cannot use a single configuration where every equipment unit is active, because the payload would overheat.
Equipment units that are on the same wall or blade of the satellite contribute to the overall temperature of that wall/blade. Therefore, we have $\ncons$ constraints, one for every set of equipment units whose thermal profiles are linked. 
For each thermal constraint $\icons \in [\ncons]$, we define a subset $\constequip{\icons} \subseteq [\nunit]$ of size $\thermalsize{\icons}$ of equipment units, among which exactly\footnote{Alternatively, one may only consider an upper bound only to prevent the system from overheating, however thermal engineers advise to keep the system as stable as possible, hence our choice of an equality.} $\thermalbound{\icons}$ should be active at the same time, i.e., in the same configuration. 
%\mjo{And, the number of equipment involved in a thermal constraint $\icons$ is denoted \thermalsize{\icons}.}

\subsection{Decisions and Objectives}

A test plan with $\nslot^*$ configurations is a mapping $\analloc$ from tests to a set of consecutive integers $[\nslot^*]$ 
(without loss of generality, we assume that configurations are numbered $1$ to $\nslot^*$),
%standing for configurations, 
and a mapping $\abufseq$ from configurations 
%$[\nslot^*]$ 
to subsets of equipment units such that the equipment units required to run a test are active when this test is run and every configuration satisfies all thermal constraints.
%
%
% \begin{itemize}
% 	\item every equipment unit required to run a test is active in the configuration this test is allocated to, and
% 	\item every configuration satisfies all thermal constraints.
% \end{itemize}
 
%There are two objectives.
The main objective is to minimize the number of configurations \toto{$|\{\alloc{\itest} \mid \itest \in [\ntest]\}|$}, in order to reduce the transition time between tests, that is, the time spent in reconfiguring the payload. %Without loss of generality, we assume that configurations are numbered $1$ to $\nslot^* = |\{\alloc{\itest} \mid \itest \in [\ntest]\}|$.
%minimize the number of times the configuration of the satellite needs to be changed.

Moreover
%, besides the number of configurations, 
it is important to take into account the total number of changes in the status of an equipment unit. Indeed, even though several equipment units can be switched on or off simultaneously, changing the status of more units requires a more careful analysis of the thermal dynamics of the system and is more likely to destabilize it.
%$\iunit$ to pairs $(\iunit,\islot)$ to a Boolean equal to \texttt{true} (1) if the equipment unit $\iunit$ is active in the configuration $\islot$.
The second %objective is the overall number of configurations \toto{$|\{\alloc{\itest} \mid \itest \in [\ntest]\}|$}, and the sequencing 
objective therefore is the total number of times an equipment unit is switched on {besides} the initial activation: $\sum_{\islot=1}^{\nslot^*}(|\bufseq{\islot} \setminus \bufseq{\islot-1}|) - \nunit$, where $\bufseq{0}$ is assumed to be empty. 

Indeed, it is important to take into account the total number of changes in the status of the equipment. Even though several equipment units can be switched on or off simultaneously, changing the status of more units requires a more careful analysis of the thermal dynamics of the system and is more likely to destabilize it.
In order to count the total number of times each equipment unit is switched from inactive to active and vice versa, we must decide the order in which the planned configurations will be visited. By convention, since configuration names are arbitrary, the tests allocated to configuration $\islot$ are run at the $\islot$-th position. Therefore, the same mapping $\analloc$ defines both the allocation of tests to configurations and the sequence in which tests will be run.

% We assume $\bufseq{\islot} = \emptyset,~\forall \islot \not\in \{\alloc{\itest} \mid \itest \in [\ntest]\}$, and since every equipment unit must be tested at least once, a constant term $\nunit$ may be subtracted from this objective. In fact, in the remainder of the paper, we will consider the number of times equipment are switched on (or off) \emph{besides} the initial (respectively, final) switch.
%

We consider here that the packing objective has higher priority than the sequencing objective and thus that they are lexicographically ordered.
%hence the optimal outcome is that giving the lexicographic least value for the pair (number of configurations, number of activations).
%
% Each of these two objectives are NP-hard to minimize.
%
%
% For the number of switches, we can setup thermal constraints so that we need a new configuration slot for every test, for instance by introducing a new equipment unit for eevery test, and a single new constraint with capacity 1 on all of these equipment units. Doing so we can model the \bfr constraint, which is NP-hard.
% Moreover, using only the packing objective, we can encode a list coloring problem. Vertices are tests, colors are configuration slots, and for every edge in the graph, we create two new equipment units and a new thermal constraint of capacity 1 between them.

\medskip

\begin{example}
Consider a set of $8$ tests and $6$ equipment units shown in Figure~\ref{sol:lex}.
The equipment required by each test is indicated with the symbol \rqmrk.
%Bold marks (\rqmrk) indicate the set of equipment units required by each test. 
Moreover, assume that we have two thermal constraints with scopes $\constequip{1} = \{1,2,3\}$ and $\constequip{2} = \{4,5,6\}$ both of capacity $2$.

%We need three configurations to run the tests in lexicographic order, as 

The solution shown in Figure~\ref{sol:lex} (\onmrk\ symbols indicate active equipment not involved in the current test) is suboptimal as it requires three configurations. Additionally, equipment units 1 and 6 must be activated twice.

However, with the permutation $2, 3, 5, 7, 8, 1, 4, 6$ shown in Figure~\ref{sol:opt}, we only need two configurations and every equipment unit is activated exactly once.

\begin{figure}
	\begin{center}
		%\begin{scriptsize}
		
		\subfloat[\label{sol:lex}]{%
		\tabcolsep=1pt
\begin{tabular}{cccc|ccc|ccc}
\multicolumn{2}{c}{conf:}& \multicolumn{2}{c}{1}& \multicolumn{3}{c}{2}& \multicolumn{3}{c}{3}\\
\multicolumn{2}{c}{test:}& 1& 2& 3& 4& 5& 6& 7& 8\\
\hline
\multirow{6}{*}{\rotatebox[origin=c]{90}{equipment}}&${1}$ & \rqmrk& \onmrk& & & & \rqmrk& \onmrk& \onmrk\\
&${2}$ & \onmrk& \rqmrk& \rqmrk& \onmrk& \rqmrk& & & \\
&${3}$ & & & \rqmrk& \onmrk& \onmrk& \onmrk& \rqmrk& \onmrk\\
\cline{2-10}
&${4}$ & \rqmrk& \onmrk& \onmrk& \rqmrk& \onmrk& & & \\
&${5}$ & & & \onmrk& \rqmrk& \rqmrk& \rqmrk& \onmrk& \rqmrk\\
&${6}$ & \onmrk& \rqmrk& & & & \onmrk& \rqmrk& \rqmrk
\end{tabular}
		}
		\hspace{.5cm}
		\subfloat[\label{sol:opt}]{%
		\tabcolsep=1pt
\begin{tabular}{ccccccc|ccc}
\multicolumn{2}{c}{conf:}& \multicolumn{5}{c}{1}& \multicolumn{3}{c}{2}\\
\multicolumn{2}{c}{test:}& 2& 3& 5& 7& 8& 1& 4& 6\\
\hline
\multirow{6}{*}{\rotatebox[origin=c]{90}{equipment}}&${1}$ & & & & & & \rqmrk& \onmrk& \rqmrk\\
&${2}$ & \rqmrk& \rqmrk& \rqmrk& \onmrk& \onmrk& & & \\
&${3}$ & \onmrk& \rqmrk& \onmrk& \rqmrk& \onmrk& \onmrk& \onmrk& \onmrk\\
\cline{2-10}
&${4}$ & & & & & & \rqmrk& \rqmrk& \onmrk\\
&${5}$ & \onmrk& \onmrk& \rqmrk& \onmrk& \rqmrk& \onmrk& \rqmrk& \rqmrk\\
&${6}$ & \rqmrk& \onmrk& \onmrk& \rqmrk& \rqmrk& & & 
\end{tabular}
		}

		%\end{scriptsize}
		\end{center}
	\caption{\label{fig:solutions} A solution with 3 configurations and 2 extra activations (\ref{sol:lex}), and an optimal solution with 2 configurations and no extra activation (\ref{sol:opt}).}
\end{figure}
	\end{example}

\subsection{Complexity}

It is relatively easy to see that the test planning problem described above is NP-hard. We show that the decision version \constraint{TestPlanning}, which asks whether there exists a plan with at most $\nslot$ configurations is NP-complete.

\begin{theorem}
	\constraint{TestPlanning} is NP-complete.
\end{theorem}

\begin{proof}
	It is in NP since a plan can be checked in polynomial time.
	
	To prove hardness, we use a straightforward reduction from \constraint{3-coloring}, which asks, given a graph $G=(V,E)$, whether there exists a coloring of $V$ with at most 3 colors such that no edge has its two end points of the same color.
	
	From a graph $G=(V,E)$, we build an instance of \constraint{TestPlanning} as follows:
	\begin{itemize}[itemsep=0pt]
		\item For every edge $(x,y) \in E$, there are two equipment units $x_y$ and $y_x$ and a thermal constraint 
		%$\iconssymb_{xy}$ 
		on these two units with capacity $1$.
		\item For every vertex $x \in V$ we create a test $t_x$ involving equipment units $x_y$ for every $y$ such that $(x,y) \in E$.
		%neighbor $y$ of $x$ in $G$.
	\end{itemize}
	
	It is easy to see that 
	%$G$ has a 3-coloring if and only if there is a plan with at most 3 configurations since 
	two tests $t_x$ and $t_y$ can share the same configuration if and only if there is no edge $(x,y) \in E$. Therefore, $G$ has a 3-coloring if and only if there is a plan with at most 3 configurations, and hence \constraint{TestPlanning} is NP-hard for $\nslot=3$.
\qed
\end{proof}

% It is easy to see that the test planning problem described above is NP-hard.
% %As the objective is hierarchical, we can
% If we consider only the number of configurations, with which we can encode a list coloring problem: vertices are tests, colors are configurations, and for every edge in the graph, we create two new equipment units and a new thermal constraint of capacity 1 between them.
Moreover, even if we let the number of configurations free, minimizing the number of switches is also NP-hard for a single thermal constraint over all equipment of capacity $\thermalbound{}$ since it corresponds exactly to the constraint:
\[
\bfr(\varalloc{},\setvar{},\thermalbound{},\equipof{},\varnumswitch{})
\] 
where $\varalloc{}$ (resp. $\setvar{}$) is a sequence of $\ntest$ integer (resp. set) variables, and for every $\itest \in [\ntest]$ the variable $\varalloc{\itest}$ has domain $[\ntest]$ and the variable $\setvar{\itest}$ ranges between the empty set and $[\nunit]$.

The number of simultaneously active equipment units must be equal to $\thermalbound{\icons}$, therefore, we can consider actived equipment as a buffer of capacity exactly $\thermalbound{\icons}$. Moreover, the set of equipment units $\equipof{\itest}$ required by a test $\itest$ corresponds to the items required to be in the buffer when processing this task. Finally, the objective is to minimize the number activations which is equivalent, up to a constant, to the number of switches $\varnumswitch{}$. The reduction from Hamiltonian path to \bfr\ provided in~\cite{DBLP:conf/cpaior/BessiereHMQW14} can be lifted to this particular case.

% The capacity $\thermalbound{\icons}$
%
%
% We need to count the number of activations of equipment units, which is neatly captured by the \switch constraint~\cite{DBLP:conf/cpaior/BessiereHMQW14}. Each thermal constraint can be seen as specifying the capacity of a buffer of active equipment units, and when a test is ran, all the units it requires must be in the buffer.
%
% The tests $\equipof{1},\ldots,\equipof{\ntest}$ are the tasks to sequence and equipment units are items to
%
%
% with one variable integer $\varalloc{\itest}$ in $[\ntest]$ and one variable $\setvar{\itest}$.

%\subsection{\mj{Example}}

% Test / Configurations / On-off
% Global constraints

\section{A Constraint Model for Test Planning}
\label{sec:mod}

% \subsection{Packing with respect to passive equipment}
% % Passive = Coloration
% % Lower bound = max clique
% % Symmetry breaking

Let $\nslot \leq \ntest$ be some known upper bound on the number of configurations. %We give a constraint satisfaction model of the test planning 
We use $\ntest$ \emph{allocation} variables $\{\varalloc{\itest} \mid \itest \in [\ntest]\}$ with domain $[\nslot]$ standing for the configuration allocated to each test.
Unless we have a valid upper bound, the maximum number of configurations $\nslot$ is equal to $\ntest$.
%We then have \mjo{$\nslot$ representing the maximum number of configurations.}
Next, we use $\nunit \times \nslot$ \emph{activity} variables $\{\varon{\iunit}{\islot} \mid \iunit \in [\nunit],~ \islot \in [\nslot]\}$ standing for the status of equipment unit $\iunit$ in configuration $\islot$, i.e., $\varon{\iunit}{\islot}$ is equal to 1 if unit $\iunit$ in switched ON in $\islot$ and is equal to 0 otherwise. Moreover, we denote $\varon{\items}{\islot}$ the set variable with characteristic function $\{\varon{\iunit}{\islot} \mid \iunit \in \items\}$. For instance, $\varon{[\nunit]}{\islot}$ has as characteristic function
%stands for
the set of variables in
the column $\islot$ of the $\nunit \times \nslot$ matrix formed by the activity variables, i.e., the set of equipment units active in configuration $\islot$.
We introduce a variable $\varnumswitch{\icons}$ for each constraint $\icons$ standing for the total number of switches on the equipment units $\constequip{\icons}$.
%Moreover, 
Finally,
we have two variables to express the objective function: a variable $\varnumconf$ standing for the number of configurations and a variable $\varnumswitch{}$ for the total number of switches. 
%Finally, we introduce a variable $\varnumswitch{\icons}$ for each constraint $\icons$ standing for the total number of switches on the equipment units $\constequip{\icons}$. % constrained by $\icons$.

\mj{To model the test planning problem, we then} post the following constraints:

\begin{eqnarray}
% \label{con:channel} \forall \itest \in [\ntest],~ \forall \iunit \in \equipof{\itest},~ &	\varon{\iunit}{\varalloc{\itest}} = 1 \\
% \label{con:thermal} \forall \icons \in [\ncons],~ \forall \islot \in [\nslot],~ & \sum_{\iunit \in \constequip{\icons}} \varon{\iunit}{\islot} = \thermalbound{\icons} \\
\label{con:channel} \forall \itest \in [\ntest],~ &	\equipof{\itest} \subseteq \varon{[\nunit]}{\varalloc{\itest}} \\
\label{con:thermal} \forall \icons \in [\ncons],~ \forall \islot \in [\nslot],~ & |\varon{\constequip{\icons}}{\islot}| = \thermalbound{\icons} \\
\label{con:confobj} \forall \itest \in [\ntest],~ & \varnumconf \geq \varalloc{\itest} \\
\label{con:switch} \forall \icons \in [\ncons],~  & \switch(\varon{\constequip{\icons}}{}, \thermalbound{\icons}, \varnumswitch{\icons}) + \thermalbound{\icons} - \thermalsize{\icons} \\
\label{con:switchobj} & \varnumswitch{} = \sum_{\icons=1}^{\ncons} \varnumswitch{\icons}  %- \sum_{\icons=1}^{\ncons}|\constequip{\icons} \cap [\nunit]| + \nunit
\end{eqnarray}

Constraints~(\ref{con:channel}) channel the allocation variables and the activity variables to ensure that every equipment unit required by a test is active in the slot in which the test is ran. They are implemented as \constraint{Element} constraints.

Constraints~(\ref{con:thermal}) enforce the thermal constraints on the number of equipment units active simultaneously within given subsets. They are implemented as simple \constraint{Sum} constraints.

Constraints~(\ref{con:confobj}) ensure that the variable representing the number of configurations is greater than or equal than the maximum allocated slot. We do not use a \constraint{Maximum} constraint nor an equality as we minimize this criterion.

Constraints~(\ref{con:switch}) count the number of switches in the sequence of set variables $\varon{\constequip{\icons}}{} = (\varon{\constequip{\icons}}{1},\ldots,\varon{\constequip{\icons}}{\nslot})$ for every thermal constraint $\icons$. 
The \switch constraint standing for the thermal constraint $\icons$ ensures that the number of switches $\sum_{\islot=1}^{\nslot} |\varon{\constequip{\icons}}{\islot} \setminus \varon{\constequip{\icons}}{\islot+1}| = \varnumswitch{\icons}$, for a capacity $\thermalbound{\icons}$ of the buffer.
The constant term $\thermalbound{\icons} - \thermalsize{\icons}$ is used to count only from the \emph{second} activation of each equipment unit.
The constraint \switch counts the number of deactivations. Moreover, the \thermalbound{\icons} items contained in the buffer at the last position are not counted by \switch, even though they will eventually be switched off. The total number of deactivations is thus $\varnumswitch{\icons} + \thermalbound{\icons}$ and it is equal to the number of activations.
Then each of the \thermalsize{\icons} equipment units constrained by $\icons$ must be activated at least once, so we can subtract this number to obtain the number of activation besides the first.
%We omit, in this model, the term \thermalbound{\icons} required to match the definition (active equipment in the last configuration must also be switched off) since it is constant.
%%We add an empty extra column $\nslot+1$ (i.e., $\varon{}{\nslot+1} = \emptyset$) to ensure that we do not count differently a unit switched off in the last configuration.   

Finally, Constraint~(\ref{con:switchobj}) computes the overall sum of switches. 
%, taking into account the equipment units covered by more than a single thermal constraint.\footnote{In the cases we have considered in this study, there is no overlap since every constraint represents the physical support of a subset of equipment units. However, overlaps may exist in some satellite architectures.}

\medskip

We consider the number of configurations as a higher priority objective than the total number of switches. Therefore, we express the objective function to minimize as the weighted sum $(\nunit\nslot/2)\varnumconf + \varnumswitch{}$.
\medskip

\mj{The Test Planning problem may be decomposed into two sub-problems:
\begin{itemize}[itemsep=0pt]
	\item a \textit{Test Packing problem} when restricting to Constraints~(\ref{con:channel}), (\ref{con:thermal}) and (\ref{con:confobj}) and the minimization of the number of configurations $\varnumconf$;
	\item a \textit{Test Sequencing problem} when restricting to Constraints~(\ref{con:channel}), (\ref{con:switch}) and (\ref{con:switchobj}) and the minimization of the total number of switches $\varnumswitch{}$.
\end{itemize}
}

For the type of satellites we have considered in this study, there is no overlap in the scope of thermal constraints since 
they stand for
%every constraint represents 
the physical support of a disjoint subset of equipment units. However, overlaps may exist in some satellite architectures. In this case, we need to replace the constraints~(\ref{con:switch}) and (\ref{con:switchobj}) by a single \switch constraint on the sequence of set variables $\varon{\nunit}{} = (\varon{[\nunit]}{1},\ldots,\varon{[\nunit]}{\nslot})$:

\begin{eqnarray}
\label{con:singleswitch} \switch(\varon{[\nunit]}{}, \sum_{\icons=1}^{\ncons}\thermalbound{\icons}, \varnumswitch{})
\end{eqnarray}

\section{Lower Bound for Test Packing}
\label{sec:pac}

The problem restricted to Constraints~(\ref{con:channel}), (\ref{con:thermal}) and (\ref{con:confobj}) has some similarities with the multi-dimensional bin packing problem~\cite{Chekuri:1999:MPP:314500.314555}. 
Instead of a one dimensional capacity, the bins have a capacity $\thermalbound{\icons}$ for each of the $\ncons$ dimension/thermal constraint.
Moreover, for a test $\itest$ requiring a set of equipment units $\equipof{\itest}$, we can compute a $\ncons$-vector representing its weight in each of these dimensions.
However, the weights are not additive since every equipment unit is activated at most once per configuration, irrespective of the number of tests requiring it in this configuration. As discussed previously, it can also be seen as a coloring problem. %, however not on a graph.
The particular case where each test is on a single equipment unit could be seen as a generalization of list coloring to hypergraphs, as thermal constraints can be mapped to hyperedges. However, to our knowledge, there is no known method for this specific problem. 

\medskip

\newcommand{\neighbor}[1]{\texmath{\Gamma_{#1}}}
\newcommand{\lbNtest}[1]{\texmath{\underaccent{\thickbar{}}{\Gamma}_{#1}}}
We therefore use a simple lower bound on the number of configurations based on the ``capacity'' of the thermal constraints.
Indeed, if a thermal constraint $\icons$ ensures that at most $\thermalbound{\icons}$ equipment units are active in a given configuration, and if the tests involve $\thermalsize{\icons}$ units, then at least
$\ceil{\thermalsize{\icons}/ \thermalbound{\icons}}$ configurations are needed. More generally, suppose that we have a lower bound $\varntested{\iunit}$ on the number of configurations in which an equipment unit $\iunit$ will be active. 
Given a thermal constraint $\icons$, the sum of these lower bounds over all equipment units in $\constequip{\icons}$ cannot be greater than the number of configurations multiplied by the capacity $\thermalbound{\icons}$, 
in other words, for any constraint $\icons$,  $\ceil{(\sum_{\iunit \in \constequip{\icons}}\varntested{\iunit})/\thermalbound{\icons}}$ is a valid lower bound for $\varnumconf$.

Now, consider an equipment unit $\iunit$, and let $\neighbor{\iunit} = \{ \iunit' \mid \exists \itest ~s.t.~ \{\iunit,\iunit'\} \subseteq \equipof{\itest}\}$ be its ``neighborhood'', that is, 
the set of units that will necessarily be active when running the tests requiring $\iunit$.
%$\neighbor{\iunit} = \{\itest \mid \iunit \in \equipof{\itest}\}$
It might not be possible to activate all these equipment units within a single configuration because of the thermal constraints. 
More generally, the equipment unit $\iunit$ must be active in at least as many configurations as are required to visit $\neighbor{\iunit}$, i.e., $\varntested{\iunit} \geq \max(\{\ceil{\frac{|\neighbor{\iunit} \cap \constequip{\icons}|}{\thermalbound{\icons}}} \mid \icons \in [\ncons]\})$, let $\lbNtest{\iunit}$ be this lower bound.

 The following two implied constraints enforce this lower bound:

\begin{eqnarray}
\label{con:lbnumtest}\forall \iunit \in [\nunit],~ & \varntested{\iunit} = \max(\lbNtest{\iunit}, \sum_{i=1}^{\nslot}{\varon{\iunit}{\islot}}) \\
\label{con:sumnumtest}\forall \icons \in [\ncons],~ & \varnumconf \geq \ceil{\frac{\sum_{\iunit \in \constequip{\icons}}\varntested{\iunit}}{\thermalbound{\icons}}} 
\end{eqnarray}

However, Constraints~\ref{con:thermal} in the base model force every set variable $\varon{\constequip{\icons}}{\islot}$ to have cardinality $\thermalbound{\icons}$, even if no test is allocated to configuration $\islot$. In this case, the configuration $\islot$ would be a copy of configuration $\varnumconf$ since it satisfies every constraint and minimize the number of switches.
However, this is not compatible with Constraint~\ref{con:lbnumtest} since equipment units active in ``dummy'' configurations would still be counted.
We therefore replace Constraints~\ref{con:thermal} with the following constraints on $\ncons\nslot$ extra Boolean variables $\{\hason{\icons}{\islot} \mid \islot \in [\nslot], \icons \in [\ncons]\}$:

\begin{eqnarray}
\label{con:channelhason} \forall \icons \in [\ncons],~ \forall \islot \in [\nslot],~ & \hason{\icons}{\islot} \iff \sum_{\iunit \in \constequip{\icons}} \varon{\iunit}{\islot} > 1 \\
\label{con:thermalconf} \forall \icons \in [\ncons],~ \forall \islot \in [\nslot],~ & \sum_{\iunit \in \constequip{\icons}} \varon{\iunit}{\islot} = \hason{\icons}{\islot} \times \thermalbound{\icons} \\
%\label{con:channelnconf}  & \unary(\varnumconf, [\hason{\icons}{1}, \ldots, \hason{\icons}{\nslot}])
\label{con:channelnconf} \forall \icons \in [\ncons],~ \forall \islot \in [\nslot],~ & \varnumconf \geq \islot \iff \hason{\icons}{\islot}
\end{eqnarray}

Constraints~\ref{con:channelhason} ensure that \hason{\icons}{\islot} equals 1 iff at least one equipment unit constrained by $\icons$ is active in configuration
$\islot$. Constraints~\ref{con:thermalconf} ensure that, for every configuration $\islot$, either there is no active equipment ($\islot$ is a dummy configuration) or
exactly $\thermalbound{\icons}$ for every thermal constraint $\icons$.
Finally, Constraints~\ref{con:channelnconf} channel these extra variables with the objective variable $\varnumconf$ by ensuring that $(\hason{\icons}{1}, \ldots, \hason{\icons}{\nslot})$ is its unary \emph{order} encoding~\cite{Tamura2008}.
%Finally, Constraint~\ref{con:channelnconf} channels these extra variables with the objective variable $\varnumconf$ by ensuring that $[\hason{\icons}{1}, \ldots, \hason{\icons}{\nslot}]$ is its unary encoding, that is: $\forall \icons \in [\ncons],~ \forall \islot \in [\nslot],~ \varnumconf \geq \islot \iff \hason{\icons}{\islot}$.
%
% \begin{eqnarray*}
% 		\forall \icons \in [\ncons],~ \forall \islot \in [\nslot-1],~ & \hason{\icons}{\islot} \geq \hason{\icons}{\islot+1} \\
% 		\forall \icons \in [\ncons],~ & \varnumconf = \sum_{\islot=1}^{\nslot} \hason{\icons}{\islot}
% \end{eqnarray*}
%These constraints are necessary to keep the
%This last constraint\footnote{This constraint is not provided in Choco, we implemen a propagator } is necessary to
These constraints are necessary to obtain a filtering 
%on thermal constraints 
as strong as in the base model, while improving the lower bound on $\varnumconf$.

\section{Test Sequencing}
\label{sec:seq}
\rque{cette section est rude .... pas tout suivi ... je dois encore relire le tout. D
ans la partie branching, tu parles ensuite de la contrainte BufferedResource, est-ce qu'un petit topo sur ces 2 contraintes ne serait pas pertinent ? }

The problem of minimizing the number of equipment activations can be naturally represented using the \switch constraint, since, as shown in Section~\ref{sec:mod}, thermal constraints and the fact that tests require some equipment units to be active can be modeled as buffered resources.
Moreover, as often in problems that can be represented with a buffered resource, we know beforehand the set of items (here equipment units) that will be activated at least once. In other words, the total number of items to be buffered minus the capacity $\thermalbound{\icons}$ of the buffer is a trivial lower bound on the required number of switches. However, the constraint \switch does not take this information into account and hence is ``suboptimal'', especially when only a few tests have been allocated a configuration.

In this section, we propose an improvement of the propagator for \switch for the very common case where we have prior knowledge on a set of items that must eventually be buffered. 
In some cases, we can simply add the number of non-buffered items to the lower bound computed by the algorithm for \switch. However, this is not always true. We define a correct lower bound based on this idea in Theorem~\ref{the:swlb}.

%Without loss of generality, we assume that the universe of values 

 We define the variant \switchp of the \switch constraint with an extra parameter $\someitems$ indicating the items that must be put in the buffer at some point in the sequence.
   Let $\someitems$ be a set of integers, $\varnumswitch{}$ an integer variable, $\setvar{} = (\setvar{1},\ldots,\setvar{\ntest})$ be a sequence of set variables, and $\mincard{},\maxcard{}$ be two sequences of $\ntest$ integers.
	 %	 and for each $\islot \in [\ntest]$, let $\setvar{\islot}$ be a set variable and $\mincard{\islot},\maxcard{\islot}$ be two integers.
\begin{definition}
  \begin{eqnarray*}
  %\switchp([\setvar{1},\ldots,\setvar{\ntest}],[\mincardinality{1}, \ldots, \mincardinality{\ntest}], [\maxcardinality{1}, \ldots, \maxcardinality{\ntest}], \someitems, \varnumswitch{}) \iff \\
	\switchp(\setvar{},\mincardinality{},\maxcardinality{}, \someitems, \varnumswitch{}) & \iff \\
  \forall \islot \in [\ntest],~ \mincardinality{\islot} \leq |\setvar{\islot}| \leq \maxcardinality{\islot} ~\wedge~ \sum_{1 \leq \islot < \ntest}|\setvar{\islot+1} \setminus \setvar{\islot}| \leq \varnumswitch{} ~\wedge~ \bigcup_{\islot=1}^{\ntest}\setvar{\islot} = \someitems &
	\end{eqnarray*}
\end{definition}

We next recall some background about the constraint and in particular its propagation algorithm in Definitions~\ref{def:nxt} and \ref{def:pri}. % for \switch works.
It is possible to find an optimal buffer sequence $\abufseq$, that is, an assignment of 
$\setvar{}$
that minimizes the value of $\varnumswitch{}$ 
with the greedy procedure \findsupport\ introduced by~\cite{DBLP:conf/cpaior/BessiereHMQW14} (Algorithm 1 in~\cite{DBLP:conf/cpaior/BessiereHMQW14}).
%with the greedy procedure \findsupport. 
This algorithm explores the sequence once while maintaining all items in a list ordered by a priority relation $\betterthan$ based on two indices ($\nextTimeInLowerBound{\anitem}{\islot}$ and  $\nextTimeNotInUpperBound{\iunit}{\islot}$) for each item $\anitem$.
%the following criteria:

\begin{definition}
	\label{def:nxt}
	$\nextTimeInLowerBound{\anitem}{\islot}$ is the least index $\islot' \geq \islot$ such that
	$\anitem \in \lbS{\islot'}$ if it exists and $\ntest+2$ otherwise.
	$\nextTimeNotInUpperBound{\iunit}{\islot}$ is the least index $\islot' \geq \islot$ such that
	$\anitem \not\in \ubS{\islot'}$ if it exists and $\ntest+1$ otherwise.
%	At index $i$ we use a priority ordering $\betterthan$ on items, based on the following criteria to satisfy (\ref{case1,case2}), applied in that order:
	% \begin{eqnarray}
	% 	\anitem \in \lbS{\islot} \cup \bufseq{\islot-1} \cap \ubS{\islot} \label{case1}\\
	% \nextTimeInLowerBound{\anitem}{\islot} < \nextTimeNotInUpperBound{\anitem}{\islot}
	% %~\&~ \nextTimeInLowerBound{\anotheritem}{\islot} > \nextTimeNotInUpperBound{\anotheritem}{\islot}
	% \label{case2}\\
	%  \nextTimeInLowerBound{\anitem}{\islot} \leq \nextTimeInLowerBound{\anotheritem}{\islot} \label{case3}\\
	% 	 \nextTimeNotInUpperBound{\anitem}{\islot} \geq \nextTimeNotInUpperBound{\anotheritem}{\islot} \label{case4}
	% \end{eqnarray}
\end{definition}

The priority relation \betterthan between two items is defined by the following criteria: 
\begin{definition}
	\label{def:pri}
At index $\islot$, and
given two items $\afirstitem < \aseconditem$, we have $\afirstitem \betterthan \aseconditem$ if:
\begin{enumerate}
	%\item $\nextTimeInLowerBound{\afirstitem}{\islot} < \nextTimeNotInUpperBound{\afirstitem}{\islot}$ and $\nextTimeInLowerBound{\aseconditem}{\islot} > \nextTimeNotInUpperBound{\aseconditem}{\islot}$, or
	
	\item $\nextTimeInLowerBound{\afirstitem}{\islot} < \nextTimeNotInUpperBound{\afirstitem}{\islot}$ and $\nextTimeInLowerBound{\afirstitem}{\islot} \leq \nextTimeInLowerBound{\aseconditem}{\islot}$, or
	
	\item $\nextTimeInLowerBound{\aseconditem}{\islot} > \nextTimeNotInUpperBound{\aseconditem}{\islot}$ and $\nextTimeNotInUpperBound{\afirstitem}{\islot} \geq \nextTimeNotInUpperBound{\aseconditem}{\islot}$

\end{enumerate}

and $\aseconditem \betterthan \afirstitem$ otherwise.
\end{definition}
	
	\medskip
	
The procedure starts with $\bufseq{0} = \emptyset$ and when moving to index $\islot$, it first adds all required items (i.e., in $\lbS{{\islot}}$)
%the lower bound of the set variable at position \islot) 
and removes all impossible items 
%(i.e., not in the upper bound at position \islot), 
(i.e., not in $\ubS{{\islot}}$), 
which yields the set $\bufseq{\islot} = \ubS{\islot} \cap \bufseq{\islot-1} \cup \lbS{\islot}$.
If $|\bufseq{\islot}| > \maxcard{\islot}$ it then removes the $\maxcard{\islot} - |\bufseq{\islot}|$ last items for the order $\betterthan$ in $\bufseq{\islot}$. Otherwise, if $|\bufseq{\islot}| < \mincard{\islot}$ it adds the $\mincard{\islot} - |\bufseq{\islot}|$ first items for the order $\betterthan$ in $\ubS{\islot} \setminus \bufseq{\islot}$.

% \medskip
%
% Now we can prove our result. The only technical difficulty is that when the optimal buffer sequence

%The principle is relatively straightforward: each item that must be visited but is not in an optimal buffer sequence will \emph{almost always} entail one extra switch. The only technicality is to show that items visited in the optimal buffer sequence that 

\begin{definition}
A stretch of a buffer sequence $\abufseq$ is a triple $\mystretch{\anitem}{a}{b}$ such that the value $\anitem$ is buffered in $\abufseq$ during the interval $[a,b]$ and $\anitem$ is not buffered at $a-1$ nor at $b+1$ (we assume that no item is buffered at 0 or $\ntest+1$).
We say that a stretch $\mystretch{\anitem}{a}{b}$ is \emph{optional} if there is no $\islot \in [a,b]$ such that $\anitem \in \lbS{\islot}$.
We say that the item $\anitem$ is optional if there exists a stretch $\mystretch{\anitem}{a}{b}$ and for all $\islot \in [\ntest]$, $\anitem \not\in \lbS{\islot}$.
\end{definition}

%The following observation is straightforward. 
Every stretch entails one switch, except if it extends until the end of the sequence. Therefore, 
if $\bufsize$ is the cardinality of the buffer at the last index of the sequence, the following observation holds:
%There must be at least one stretch for each item in the first buffer, and every switch entails exactly one new stretch.

\begin{observation}
A solution with $\numswitch$ switches has $\bufsize+\numswitch$ stretches.
\end{observation}

\begin{lemma}
	\label{lem:singlechange}
	If $\abufseq$ and $\abufseq'$ are two buffer sequences found by \findsupport on two instances $\aninstance$ and $\anotherinstance$ equal on every index except one for which the lower bound of the buffer at this index contains a single additional item $\anitem$ in $\anotherinstance$,
	then 
	%$\forall \islot \in [\ntest],~ |\bufseq{\islot} \setminus \otherbufseq{\islot}| \leq 1$ and
	$\forall \islot \in [\ntest],~ \otherbufseq{\islot} \setminus \bufseq{\islot} \subseteq \{\anitem\}$.
\end{lemma}

\begin{proof} 
An item not in the buffer at index $\islot-1$ is added at index $\islot$ only if:
	\begin{enumerate}
		\item the item is the lower bound $\lbS{\islot}$ or \label{inlb}
		%\item $\mincard{\islot}>|\bufseq{\islot-1}|$ and the item is in the $\mincard{\islot}-|\bufseq{\islot-1}|$ first for $\betterthan$,
		%\item or $\maxcard{\islot}<|\bufseq{\islot-1} \cap \ubS{\islot}|$ and the item is in the $|\bufseq{\islot-1} \cap \ubS{\islot}|-\maxcard{\islot}$ first for $\betterthan$.
		\item $\mincard{\islot}>|\ubS{\islot} \cap \bufseq{\islot-1} \cup \lbS{\islot}|$ and it is in the $|\ubS{\islot} \cap \bufseq{\islot-1} \cup \lbS{\islot}|-\maxcard{\islot}$ first for $\betterthan$. \label{moreprio}
	\end{enumerate}
	
	Only item $\anitem$ can satisfies case~\ref{inlb} in $\anotherinstance$ but not in $\aninstance$.
	
Moreover, by definition, the order $\betterthan$ is equal in $\aninstance$ and $\anotherinstance$, except for $\anitem$ which may be ranked higher in $\anotherinstance$. 
Therefore, the only item that may satisfy case~\ref{moreprio} in $\anotherinstance$ but not in $\aninstance$ is again $\anitem$. Therefore, $\forall \islot \in [\ntest],~ \otherbufseq{\islot} \setminus \bufseq{\islot} \subseteq \{\anitem\}$.
	\qed
\end{proof}

\begin{lemma}
		\label{lem:nobridge}
		Let $\abufseq$ be a buffer sequence of an instance $\aninstance$ with minimal number of switches and let $\anitem$ be an item not buffered in \abufseq.
		For any $\islot \in [\ntest]$, a sequence \anotherbufseq on the instance $\anotherinstance$ obtained by adding the constraint $\anitem \in \lbS{\islot}$ 
		has at least one more non-optional stretch than \abufseq.
\end{lemma}

\begin{proof}
	% We know that \anotherbufseq has a non-optional stretch \mystretch{\anitem}{a}{b} that was not in \abufseq. Therefore, it must have one less non-optional stretch on an item $\anotheritem \neq \anitem$, and this can only happen if the gap between two non-optional stretches of \anotheritem is bridged, that is, by buffering the value \anotheritem when it is not buffered in \abufseq. However, by Lemma~\ref{lem:singlechange} we have $\forall \islot \in [\ntest],~ \otherbufseq{\islot} \setminus \bufseq{\islot} \subseteq \{\anitem\}$
  %
	We can assume that \abufseq and \anotherbufseq were found by \findsupport as this algorithm is complete.
	The sequence \anotherbufseq has necessarily a non-optional stretch \mystretch{\anitem}{a}{b},
	and there was no \anitem-stretch in \abufseq. Therefore, if the number of non-optional stretches is not larger in \anotherbufseq than in \abufseq, it must have one less non-optional \anotheritem-stretch for an item $\anotheritem \neq \anitem$. This can only happen if the gap between two non-optional stretches \mystretch{\anotheritem}{a_1}{b_1} and \mystretch{\anotheritem}{a_2}{b_2} with $b_1+1<a_2$
	%stretches of \anotheritem 
	is bridged by buffering the value \anotheritem in the interval $[b_1+1,a_2-1]$. %, in other words, there exists $\islot$ such that 
	Hence there exists $\islot$ such that $\anotheritem \in \otherbufseq{\islot} \setminus \bufseq{\islot}$.
	%when it is not buffered in \abufseq. 
	%Moreover, there exists such a solution that can be found by \findsupport since this algorithm is correct.
	However, by Lemma~\ref{lem:singlechange} we have $\forall \islot \in [\ntest],~ \otherbufseq{\islot} \setminus \bufseq{\islot} \subseteq \{\anitem\}$. 
\qed
\end{proof}

\begin{theorem}
	\label{the:swlb}
	For any two sets $\someitems \subseteq \items$,
if there is an optimal buffer sequence visiting exactly the items in $\items \setminus \someitems$ with $\numswitch$ switches, $\numstretch$ optional stretches, and $\numopt$ optional items, then
there is no buffer sequence visiting all items in $\items$ in less than $\numswitch+|\someitems|-\numstretch+\numopt$ switches.
\end{theorem}

\begin{proof}

First, notice that we can reduce the case with $\numopt>0$ optional items to the case without optional item, and $\numstretch - \numopt$ optional stretches. Indeed, if there exists an optimal sequence \abufseq visiting all items in $\items$, we can add a constraint $\anitem \in \lbS{\islot}$
for every pair $(\islot,\anitem)$ where $\anitem \in \items$ and $\anitem \in \bufseq{\islot}$.
The procedure \findsupport will then find a sequence with same
 %the \numopt optional items without changing the 
 number of switches, 
 %nor the 
 same number of stretches, 
 %and only reducing the number of 
 and \numopt less
 optional stretches than \abufseq. Therefore, we suppose $\numopt=0$ and prove the lower bound $\numswitch+|\someitems|-\numstretch$.

\medskip

Now, if \abufseq has $\bufsize+\numswitch-\numstretch$ non-optional stretches, then by Lemma~\ref{lem:nobridge}, we know that a solution visiting all items in \items must have at least $\bufsize+\numswitch+|\someitems|-\numstretch$ (non-optional) stretches, and hence at least $\numswitch+|\someitems|-\numstretch$ switches.
\qed
\end{proof}

Counting the number of optional stretches and items in the sequence returned by \findsupport can be done in linear time. Therefore, Theorem~\ref{the:swlb} improves the lower bound found by this algorithm without changing its worst case complexity when we know that some items must be buffered but do not appear in the lower bound of a set variable. This is true in the test sequencing problem, as it is in most applications of this constraint.

\section{Search Strategy}
\label{sec:str}

In this section we introduce a dedicated branching heuristic for the packing problem. The basic idea is that we can easily evaluate the impact of allocating a test to a configuration by counting how many equipment activations are required and how these equipment units are already constrained in this configuration.
Second, we propose to decompose the problem into packing and sequencing aspects in order to find good upper bounds quickly for the test planning problem.

\newcommand{\deltacons}[3]{\texmath{\delta_{#1}^{#3}(#2)}}
\newcommand{\impactcons}[3]{\texmath{\gamma_{#1}^{#3}(#2)}}
\newcommand{\impactdec}[2]{\texmath{\gamma(\varalloc{#1} \gets #2)}}
\subsection{\mj{Branching Heuristic}}
%\subsection{Search Heuristic}
\label{ssec:heur}

Let $\deltacons{\itest}{\islot}{\icons}$ be the number of  equipment units constrained by $\icons$ that will be active in configuration $\islot$ if test $\itest$ was to be run in that configuration, i.e., the number of non-ground Boolean activation variables concerning equipment of test \itest constrained by \icons:
$$
\deltacons{\itest}{\islot}{\icons} = \sum_{\iunit \in \constequip{\icons} \cap \equipof{\itest}} \ubON{\iunit}{\islot} - \lbON{\iunit}{\islot}
$$ 

We consider the ratio $\frac{\thermalsize{\icons}}{\thermalbound{\icons}}$ to be proportional to the tightness of the constraint $\icons$, and 
we use the change in tightness resulting from the decision of running test  $\itest$ in configuration $\islot$ to evaluate the impact of this decision.
After the decision $\varalloc{\itest} = \islot$, the tightness $r_b = \frac{\thermalsize{\icons}}{\thermalbound{\icons}}$ becomes 
$r_a = \frac{\thermalsize{\icons}-\deltacons{\itest}{\islot}{\icons}}{\thermalbound{\icons}-\deltacons{\itest}{\islot}{\icons}}$.
Therefore, the factor $r_a/r_b$ represents the factor by which the tightness of constraint $\icons$ would increase.
As $r_a/r_b \in [1,\infty]$, we use 1 minus its inverse as a measure, in $[0,1]$, of the impact of that decision, that is:
%
% then we can use the increase factor $\impactcons{\itest}{\islot}{\icons}$ on the tightness of $\icons$ of running test  $\itest$ in slot $\islot$ as a measure of the impact of that decision.
% In fact we use 1 minus the decrease factor, so that we get a real value in $[0,1]$ (we have $\deltacons{\itest}{\islot}{\icons} \in [0,\thermalbound{\icons}]$ if we assume that the channeling between allocation and activity variables is properly maintained).
% %to represent the impact, with respect to constraint $\icons$ of running test  $\itest$ in slot $\islot$.
$$
\impactcons{\itest}{\islot}{\icons} = 1 - \frac{\thermalsize{\icons}(\thermalbound{\icons}-\deltacons{\itest}{\islot}{\icons})}{\thermalbound{\icons}(\thermalsize{\icons}-\deltacons{\itest}{\islot}{\icons})}
$$

Now, we can use the average of these impacts on all thermal constraints to define the impact $\impactdec{\itest}{\islot}$ of allocating test $\itest$ to configuration $\islot$.

$$
\impactdec{\itest}{\islot} = \frac{\sum_{\icons \in [\ncons]} \impactcons{\itest}{\islot}{\icons}}{\ncons} 
$$

Notice that allocation variables may have many more values than necessary to satisfy the thermal constraints. Indeed, it is important, for a coloring heuristic, to branch only values in $[\islot+1]$ where $\islot$ is the highest allocated value so far. Hence, given a test $\itest$ we consider the intersection $\domcol{\itest} = \dom{\varalloc{\itest}} \cap [\islot+1]$ instead of its actual domain $\dom{\varalloc{\itest}}$.
We therefore select
%chose, for the next decision, 
the variable $\varalloc{\itest}$ minimizing:
$$
\frac{|\domcol{\itest}|}{\sum_{\islot \in \domcol{\itest}}\impactdec{\itest}{\islot}}
$$
And we branch on the value $\islot$ minimizing $\impactcons{\itest}{\islot}{\icons}$.

%Finally, as in~\cite{shaw} use the domain size divided by the average impact to select the variable with highest overall impact

\subsection{Multi-stage Approach}
\label{ssec:multi}

We use $\nunit\nslot$ Boolean variables to represent the status of every equipment unit in every configuration. Moreover, without an upper bound on the number of configurations required to pack every test, we can only assume that $\ntest$ configurations (as many as there are tests) may be needed, i.e., $\nslot=\ntest$. However, in practice \ntest is a gross overestimate of \nslot. For instance, in one of the industrial instances that we considered, several hundreds of tests can be run in as few as three configurations. 

Furthermore, since we consider two criteria in a hierarchical way, it makes sense to optimize a relaxation of the problem where only Constraints~\ref{con:channel} to \ref{con:confobj} are kept \mj{(the packing sub-problem)}, i.e., the model is complete with respect to the objective with highest priority.
%We refer to this relaxed model as the packing model. 
Observe that since the order of the bins does not matter in this case, 
configurations are symmetric. 
%we can now break some symmetries, for instance by using 
We therefore used 
lexicographic ordering constraints~\cite{lexchain,Frisch2006803} on the set variables $\varon{}{1}, \ldots, \varon{}{\nslot}$.

Last, we also considered the pure sequencing aspect of the problem. Given a packing, finding the optimal sequence for that packing can be modeled with a set of \bfr constraints, one for each thermal constraint, all sharing the same permutation. 
Another way to understand this is that we can consider a packing solution using $\nslot$ configurations as a new instance with only $\nslot$ tests to sequence (it is unlikely that two such tests could share a single configuration given the thermal constraints) and consider Constraints~\ref{con:channel}, \ref{con:switch} and \ref{con:switchobj} of the complete model \mj{(the sequencing sub-problem)}. 
Solving this problem to optimality does not give us a lower bound on the total number of switches, however, it is much simpler and can often provide a good upper bound quickly.

We therefore implemented the following four-phase strategy:
\vspace{-2pt}
\begin{enumerate}
	\item We run a greedy descent on the packing problem to find an initial upper bound. The trivial heuristic that branch on the lexicographically least configuration for a test gives relatively good results, so we stop this phase at the first solution found, which is backtrack-free.
	\item We run the packing model (initialized with the previous upper bound) for a given period of time, or until optimality is proven.
	\item We run the sequencing model for a given period of time, or until optimality is proven (though in this case we cannot deduce a lower bound).
	\item We run the complete \mj{model} (packing \& sequencing) (initialized with lower and/or upper bounds, accordingly) for the rest of the allocated time.
\end{enumerate}

\def\base{\algo{base}}
\def\prop{\algo{propagation}}
\def\lba{\algo{lb conf}}
\def\bfr{\algo{lb switch}}
\def\heur{\algo{heuristic}}
\def\decomp{\algo{multi-stage}}
\def\strat{\algo{choco+strat}}
\def\full{\algo{full}}
\def\fullnoconf{\algo{\full $\setminus$ lb conf.}}
\def\fullnoswitch{\algo{\full $\setminus$ lb switch}}
\def\NCONF{\#conf}
\def\NSWITCH{\#switch}

\section{Experimental Evaluation}
\label{sec:exp}

%\subsection{Test Planning}

%\newcommand{\phaseenlight}[1]{#1}
\newcommand{\phaseenlight}[1]{\textsc{#1}}
We tested the different approaches that we propose on industrial and generated instances. We have only six real instances, corresponding to three already launched communications satellites.
The same tests are usually run in two types of thermal environments.
The
\phaseenlight{hot} and \phaseenlight{cold} test phases respectively 
simulate the periods where the satellite is facing the sun, or when it is in Earth's shadow.
% with two types of test campaign in each case: either in a ``cold'' setting, which implies relatively tight thermal constraints, and a ``hot'' setting in which thermal constraints are slightly relaxed. We randomized these instances [by shuflling  MJO?].
Instances labeled \instance{cold} are much more thermally constrained and thus typically require more configurations than those labeled \instance{hot}.
We diversified the pool of instances by randomly shuffling tests in order to produce 5 randomized variants of every instance.
Moreover, we used random instances designed to be similar to the real cases,
% in order to better control their size and hardness.  
%These instance were 
generated as follows:
for a given number of tests $\ntest \in \{30, 50, 80, 100, 200, 300 \}$, we set the number of equipment units to $\ntest/4$. 
%Equipment units are partitioned between $3$ (for the two smaller classes) or $5$ (for other classes) thermal constraints.
\toto{The equipment of instances with $30$ and $50$ tests are equally partitioned into $3$ thermal constraints. Each test requires $2$ equipment units from two different thermal constraints.
%the first unit belonging to one of the two first constraints and the other unit belonging to the last constraint.
Other instances have $5$ thermal constraints and each test requires $2$ or $3$ equipment units belonging to different thermal constraints.}
%, one unit belonging to one of the three first constraints and the other units belonging, respectively, to the two last constraints.
%In all instances, each equipment unit is associated to a unique thermal constraint.}
Then, for each generated instance, we consider two levels of tightness for thermal constraints to simulate the \phaseenlight{hot} (\thermalbound{\icons}/\thermalsize{\icons}=0.6) and the \phaseenlight{cold} phases (\thermalbound{\icons}/\thermalsize{\icons}=0.4).
%In the \phaseenlight{cold} phase (resp. \phaseenlight{hot}), thermal constraint allow $40\%$ (resp. $60\%$) of the equipment units they involve to be simultaneously active.
 % Thermal constraints allowing a high number of simultaneously activated equipment are linking to the \phaseenlight{hot} phase with $60\%$ of equipment for each constraint that can be ON. For the \phaseenlight{cold} phase, this level is limited to $40\%$ leading to a most constrained problem.
 
 %\input{table.tex} 

We generated 5 variants of 12 classes of instances. Instance \instance{XXX.YY} denote a set of \instance{XXX} tests with \instance{YY} giving the 
ratio $\thermalbound{\icons}/\thermalsize{\icons}$ 
%level 
of the thermal constraints. Instances \instance{A}, \instance{B} and \instance{C} are industrial instances.
All models have been implemented in Choco 3~\cite{choco} and ran on Intel Xeon E5 processors for a total of thirty minutes on every instance.
We compared the five following approaches:
\begin{itemize}[itemsep=4pt]
	\item \base is the straightforward model with Constraints~\ref{con:channel} to \ref{con:switchobj}, using \textit{Weighted Degree}~\cite{04-wdeg} for variable selection and lexicographic branching. 
	Notice that the strategy is set up to branch on all allocation variables before branching on activity variables. Other predefined heuristics in Choco were less efficient. 
	It is important to note that due to the ``coloring'' aspect of the problem, branching on the lexicographically least value (color) in the domain is extremely important. 
	However \emph{Impact Based Search}~\cite{Refalo2004} and \emph{Activity Based Search}~\cite{Michel2012} cannot trivially be made to branch on the lexicographically least value and thus gave extremely poor results for that reason.
	
	\item \heur is the same %basic 
	model as \base, however with the dedicated heuristic described in Section~\ref{ssec:heur} for variable ordering.
	
	\item \decomp is the same 
	%basic 
	model as \base, however, using the \mjo{four}-phase strategy described in Section~\ref{ssec:multi}.
	
	\item \prop augments the model \base with symmetry breaking, the lower bound on the packing defined by Constraints~(\ref{con:lbnumtest}) to (\ref{con:channelnconf}), and the improvement on the propagator for \switch described in Section~\ref{sec:seq}.
	
	% \item \lba augments \base with the lower bound on the packing defined by Constraints~\ref{con:lbnumtest} to \ref{con:channelnconf}.
	% %symmetry breaking and
	%
	% \item \bfr augments \base with the improvement on the propagator for \switch described in Section~\ref{sec:seq}.
	%
	\item \full \mj{is the same model as \prop}, however it uses the \mjo{four-phases} strategy and the branching heuristic described in Section~\ref{sec:str}.
	
\end{itemize}

Note that for every approach, we first applied a preprocessing to the data in order to merge identical tests. Indeed, on top of the packing based on the configuration of active equipment units, one must, in the real setting, further partition the tests because of different requirements on signal routing equipment. This second packing phase is exactly a list-coloring problem on the tests of each the configurations. The size of these problems is very modest with respect to state of the art coloring algorithms, so we do not study this aspect in the current paper. As a consequence, some tests \mj{in the data} require the same set of equipment units to be active, and can be thought of as a single test for our purpose.  

% \clearpage
% \input{table10m.tex}
\input{table30m.tex}
% \clearpage

% \input{table30m.tex}

The results are reported in Table~\ref{tab:exp}. For each method, we show the objective values, where \textit{\NCONF{}} stands for the average number of configurations over the five variants of the instance, and \textit{\NSWITCH{}} stands for the average number of switches for the same instances. Next to these values, we report in brackets, the number of instances for which the reported objective value was proven optimal within a 30 minutes time limit (no value means that none of the runs was complete). 
For each instance, the methods proving optimality in the most cases, among those giving the minimum average objective value, are color-highlighted.

First, we observe that the straightforward model \base is very poor, in most of the larger instances it does not even find a feasible solution in less than ten minutes.
On the other hand, augmenting this model either with stronger propagation, better branching heuristic, or the multi-stage approach is sufficient to obtain decent results. Combining all these improvements clearly yields the best and most robust results, and sometimes allows to prove optimality even on real industrial instances. 

Second, multi-stage approaches (\decomp and \full) are clearly better for larger instances, however, they tend to hinder the ability of Choco to prove optimality on small instances. Moreover, we observed that they tend to be less robust, in the sense that the final result greatly depends on the initial packing \rque{obtained in the second phase ?} which gives no guarantees on the number of switches. \rque{comment on lit cela sur le tableau ?}
This is especially true for large instances in which the complete model often cannot improve on the heuristic sequence found during the third phase.

\medskip

\def\airbus{\algo{CM}}
\def\airbustsp{\algo{CM+TSP}}

\input{tablefactor30m.tex}

For a deeper analysis of the two bounds proposed in this paper, we ran four other models on the same instances. The first two are the base model augmented with the lower bound on configurations or the lower bound on switches 
(\base $\oplus$ \lba and \base $\oplus$ \bfr, respectively).
%(\lba and \bfr, respectively). 
The other two models are the full model from which we removed these lower bounds (\full $\setminus$ \bfr and \full $\setminus$ \lba, respectively).
%(\fullnoconf and \fullnoswitch, respectively). 
The results of these additional tests, in Table~\ref{tab:factor} clearly show that both bounds are useful. Surprisingly, the capacity to prove optimality on a criterion is also impacted by the bound on the other criterion.
%The impact of symmetry breaking, however, is much less obvious. %, perhaps because it applies only to the packing sub-problem.  

%
% %Moreover, both lower bounds (for the number of configurations and for the number of switches) have an overall similar impact.
%
%
%
% [TODO:]
% \begin{itemize}
% 	%\item Comparison with Caroline
% 	\item Full Model without \switch improvement or without packing lower bound is equivalent for the number of configurations, slightly less efficient 68 vs 69 for number of switch and for the number of proofs (74 vs 66 vs 64), but better than the multi-stage (not as good on \NCONF, 69.7 \NSWITCH, 40 proofs)
% 	\item Randomness factor in the multi-stage approaches because the pure sequencing phase is an approximation
% 	\item multi-stage approach good for large instances, but hinders the ability to prove optimality on small instances
% 	\item proofs for \NSWITCH only when 0 ``re-switch'' is possible
% \end{itemize}

\medskip
Finally, in Table~\ref{tab:CM}, we compare the results of our best method (model \full) with the method previously used by Airbus Defense \& Space~\cite{maillet2011constraint}  and with a slight improvement of this method described in \cite{bochesauvan:hal-01166683}. The former, denoted \airbus, is a constraint optimization tool build to solve the packing problem only. 
The second, denoted \airbustsp, is the same approach, however the resulting configurations are then permuted so that the \emph{Traveling Salesman Problem} defined by the Hamming distance between configurations is optimized. It is not surprising that the first method is very poor in terms of total equipment activation since it makes no attempt to optimize this criterion. However, it is interesting to compare with our results as it is still the method used in practice. \rque{est-ce qu'on va avoir le droit de dire cela ?}

%In order to better visualize the results, instead of the actual number of switches, i.e., how many times an equipment unit is switched ON, or equivalently OFF, taking into account, respectively the initial or final switch, we report the number of \emph{extra} switches, that is, we substract the number of equipment units since every equipment must be switched ON at least once.
The method \mj{\airbustsp}
%described in \cite{bochesauvan:hal-01166683} 
gives a better solution in one case (\instance{C cold}). 
Indeed, this instance is particular as it involves more than twice as many tests as other instances. The consequence is that the model computing packing and sequencing simultaneously is relatively inefficient. 
\mjo{For this instance,} we therefore ran a simple randomized sequence of the second and third phases of the multi-stage approach (i.e., pure packing followed by pure sequencing) and quickly found a similar solution (though we could not improve on it).
Notice that in the case of the ``hot'' test campaign for satellite \instance{C}, carefully packing and sequencing the tests makes it possible to get rid of all equipment activation besides the mandatory one, whereas 14 activations are necessary with the current method.

\input{tableCM.tex}

\section{Conclusion}
\label{sec:con}

We have introduced a complete constraint programming approach for the problem of packing and sequencing the validation tests of communications satellites. 
We proposed a search strategy and lower bound for the packing aspect of the problem. Moreover, we introduced an improvement of the \switch constraint that can be applied in many other contexts. 
Our experimental evaluation shows that the methods proposed in this paper greatly improve 
the test plans with respect to those 
%the method 
currently used within the Airbus group.
% to solve this problem. 

Although this approach is not yet industrially implemented, a previous \mj{internal study in Airbus}\footnote{Master's internship report by Ludivine Boche-Sauvan for the ``Institut Sup\'erieur de l'Aeronautique et de l'Espace'' (ISAE) in 2012\\ (\url{http://www.laas.fr/files/ROC/LAAS_Techreport.pdf}).}
%field study within Airbus 
has shown that, 
\mj{during a test campaign, around 30\% of the total duration is spent on transitions between configurations.} Moreover, in many cases, tests must be interrupted because the payload is overheating and can only resume after the system has been stabilized. The constraint model we introduced should help with both of these issues.

%the total duration of the tests could be reduced by up to 30\% with such optimized plans. 
Since such \mj{test campaigns}
%tests require to run 
require an extremely costly and energy greedy thermal vacuum chamber as well as a large team of engineers in 3-shift rosters, significant financial savings are expected from this approach. 

%\clearpage

\bibliographystyle{plain}
\bibliography{biblio}

\end{document}

%% file: table30m.tex
\begin{table}[t]
\caption{\label{tab:exp}Methods comparison. Average number of configurations and switches on random and industrial instances.}
\tabcolsep=1pt
\begin{center}
\begin{scriptsize}
\begin{tabular}{l |rlrl |rlrl |rlrl |rlrl |rlrl }
\multirow{2}{*}{instance} & \multicolumn{4}{|c}{ \base } & \multicolumn{4}{|c}{ \heur } & \multicolumn{4}{|c}{ \decomp } & \multicolumn{4}{|c}{ \prop } & \multicolumn{4}{|c}{ \full } \\
& \multicolumn{2}{|c}{\NCONF} & \multicolumn{2}{c}{\NSWITCH}  & \multicolumn{2}{|c}{\NCONF} & \multicolumn{2}{c}{\NSWITCH}  & \multicolumn{2}{|c}{\NCONF} & \multicolumn{2}{c}{\NSWITCH}  & \multicolumn{2}{|c}{\NCONF} & \multicolumn{2}{c}{\NSWITCH}  & \multicolumn{2}{|c}{\NCONF} & \multicolumn{2}{c}{\NSWITCH}  \\
\hline
\hline
\instance{ 030$\cdot$04 } & \cellcolor{TealBlue!30}{6.4} & \cellcolor{TealBlue!30}{(5)} & \cellcolor{TealBlue!30}{1.6} & \cellcolor{TealBlue!30}{(5)} & \cellcolor{TealBlue!30}{6.4} & \cellcolor{TealBlue!30}{(5)} & \cellcolor{TealBlue!30}{1.6} & \cellcolor{TealBlue!30}{(5)} & \cellcolor{TealBlue!30}{6.4} & \cellcolor{TealBlue!30}{(5)} & 1.6 & (4) & \cellcolor{TealBlue!30}{6.4} & \cellcolor{TealBlue!30}{(5)} & \cellcolor{TealBlue!30}{1.6} & \cellcolor{TealBlue!30}{(5)} & \cellcolor{TealBlue!30}{6.4} & \cellcolor{TealBlue!30}{(5)} & 1.6 & (4) \\
\instance{ 030$\cdot$06 } & \cellcolor{TealBlue!30}{3.4} & \cellcolor{TealBlue!30}{(5)} & \cellcolor{TealBlue!30}{0.8} & \cellcolor{TealBlue!30}{(5)} & \cellcolor{TealBlue!30}{3.4} & \cellcolor{TealBlue!30}{(5)} & \cellcolor{TealBlue!30}{0.8} & \cellcolor{TealBlue!30}{(5)} & \cellcolor{TealBlue!30}{3.4} & \cellcolor{TealBlue!30}{(5)} & 0.8 & (3) & \cellcolor{TealBlue!30}{3.4} & \cellcolor{TealBlue!30}{(5)} & \cellcolor{TealBlue!30}{0.8} & \cellcolor{TealBlue!30}{(5)} & \cellcolor{TealBlue!30}{3.4} & \cellcolor{TealBlue!30}{(5)} & 0.8 & (2) \\
\instance{ 050$\cdot$04 } & 5.2 & (2) & 2.2 & (2) & 5.2 & (3) & 2.0 & (3) & \cellcolor{TealBlue!30}{5.2} & \cellcolor{TealBlue!30}{(5)} & 2.0 & (3) & \cellcolor{TealBlue!30}{5.2} & \cellcolor{TealBlue!30}{(5)} & \cellcolor{TealBlue!30}{2.0} & \cellcolor{TealBlue!30}{(5)} & \cellcolor{TealBlue!30}{5.2} & \cellcolor{TealBlue!30}{(5)} & \cellcolor{TealBlue!30}{2.0} & \cellcolor{TealBlue!30}{(5)} \\
\instance{ 050$\cdot$06 } & 3.6 & (2) & 0.8 & (2) & \cellcolor{TealBlue!30}{3.6} & \cellcolor{TealBlue!30}{(5)} & \cellcolor{TealBlue!30}{0.6} & \cellcolor{TealBlue!30}{(5)} & \cellcolor{TealBlue!30}{3.6} & \cellcolor{TealBlue!30}{(5)} & 0.6 & (3) & \cellcolor{TealBlue!30}{3.6} & \cellcolor{TealBlue!30}{(5)} & \cellcolor{TealBlue!30}{0.6} & \cellcolor{TealBlue!30}{(5)} & \cellcolor{TealBlue!30}{3.6} & \cellcolor{TealBlue!30}{(5)} & 0.6 & (2) \\
\instance{ 080$\cdot$04 } & 9.6 &  & 19.0 &  & 9.0 & (1) & 15.0 & (1) & 9.0 &  & 15.4 &  & 9.0 & (1) & \cellcolor{TealBlue!30}{13.8} & \cellcolor{TealBlue!30}{(1)} & \cellcolor{TealBlue!30}{9.0} & \cellcolor{TealBlue!30}{(3)} & \cellcolor{TealBlue!30}{13.8} & \cellcolor{TealBlue!30}{(1)} \\
\instance{ 080$\cdot$06 } & 5.6 &  & 6.0 &  & 5.4 & (3) & 4.8 & (3) & 5.4 & (3) & 5.2 & (1) & 5.4 & (3) & \cellcolor{TealBlue!30}{4.2} & \cellcolor{TealBlue!30}{(3)} & \cellcolor{TealBlue!30}{5.4} & \cellcolor{TealBlue!30}{(4)} & \cellcolor{TealBlue!30}{4.2} & \cellcolor{TealBlue!30}{(3)} \\
\instance{ 100$\cdot$04 } & 14.0 &  & 40.6 &  & 13.0 &  & 43.4 &  & 13.2 &  & 31.8 &  & 13.0 &  & 34.4 &  & \cellcolor{TealBlue!30}{12.8} &  & \cellcolor{TealBlue!30}{30.2} &  \\
\instance{ 100$\cdot$06 } & 4.4 &  & 5.8 &  & 4.0 & (4) & 2.8 & (4) & \cellcolor{TealBlue!30}{4.0} & \cellcolor{TealBlue!30}{(5)} & 4.0 &  & 4.0 & (4) & 3.0 & (4) & \cellcolor{TealBlue!30}{4.0} & \cellcolor{TealBlue!30}{(5)} & \cellcolor{TealBlue!30}{2.8} & \cellcolor{TealBlue!30}{(5)} \\
\instance{ 200$\cdot$04 } & 12.6 &  & 67.2 &  & 10.2 &  & 62.8 &  & 10.2 &  & \cellcolor{TealBlue!30}{46.2} &  & 10.4 &  & 53.6 &  & \cellcolor{TealBlue!30}{10.0} &  & 47.8 &  \\
\instance{ 200$\cdot$06 } & 5.8 &  & 24.2 &  & \cellcolor{TealBlue!30}{5.0} &  & 18.2 &  & \cellcolor{TealBlue!30}{5.0} &  & 17.2 &  & \cellcolor{TealBlue!30}{5.0} &  & \cellcolor{TealBlue!30}{14.0} &  & \cellcolor{TealBlue!30}{5.0} &  & 14.4 &  \\
\instance{ 300$\cdot$04 } & - &  & - &  & 10.4 &  & 111.2 &  & 10.2 &  & 85.0 &  & 10.2 &  & 94.0 &  & \cellcolor{TealBlue!30}{10.0} &  & \cellcolor{TealBlue!30}{81.8} &  \\
\instance{ 300$\cdot$06 } & 5.7 &  & 45.3 &  & 4.8 &  & 37.0 &  & 4.6 &  & 26.8 &  & 4.2 &  & 25.4 &  & \cellcolor{TealBlue!30}{4.0} &  & \cellcolor{TealBlue!30}{22.4} &  \\
\hline
\instance{ A cold } & \cellcolor{TealBlue!30}{6.0} &  & 16.0 &  & \cellcolor{TealBlue!30}{6.0} &  & 14.2 &  & \cellcolor{TealBlue!30}{6.0} &  & 8.8 &  & \cellcolor{TealBlue!30}{6.0} &  & 9.6 &  & \cellcolor{TealBlue!30}{6.0} &  & \cellcolor{TealBlue!30}{7.0} &  \\
\instance{ A hot } & 4.0 &  & \cellcolor{TealBlue!30}{2.0} &  & 4.0 &  & 4.6 &  & 4.0 &  & \cellcolor{TealBlue!30}{2.0} &  & 4.0 &  & \cellcolor{TealBlue!30}{2.0} &  & \cellcolor{TealBlue!30}{4.0} & \cellcolor{TealBlue!30}{(5)} & \cellcolor{TealBlue!30}{2.0} &  \\
\instance{ B cold } & - &  & - &  & \cellcolor{TealBlue!30}{9.0} &  & 106.6 &  & \cellcolor{TealBlue!30}{9.0} &  & \cellcolor{TealBlue!30}{50.2} &  & \cellcolor{TealBlue!30}{9.0} &  & 78.8 &  & \cellcolor{TealBlue!30}{9.0} &  & \cellcolor{TealBlue!30}{50.2} &  \\
\instance{ B hot } & - &  & - &  & 4.0 &  & 19.6 &  & 4.0 &  & 3.6 &  & 4.0 &  & 1.2 &  & \cellcolor{TealBlue!30}{4.0} & \cellcolor{TealBlue!30}{(5)} & \cellcolor{TealBlue!30}{0.8} & \cellcolor{TealBlue!30}{(2)} \\
\instance{ C cold } & - &  & - &  & \cellcolor{TealBlue!30}{8.0} &  & 82.8 &  & \cellcolor{TealBlue!30}{8.0} &  & \cellcolor{TealBlue!30}{46.8} &  & \cellcolor{TealBlue!30}{8.0} &  & 64.8 &  & \cellcolor{TealBlue!30}{8.0} &  & 50.6 &  \\
\instance{ C hot } & - &  & - &  & 3.0 &  & 15.0 &  & 3.0 &  & 2.0 &  & 3.0 & (4) & 0.2 & (4) & \cellcolor{TealBlue!30}{3.0} & \cellcolor{TealBlue!30}{(5)} & \cellcolor{TealBlue!30}{0.0} & \cellcolor{TealBlue!30}{(5)} \\
\hline
\hline
\end{tabular}
\end{scriptsize}
\end{center}
\end{table}

%% file: tablefactor30m.tex
\begin{table}[t]
\caption{\label{tab:factor}Impact of the lower bounds. Average number of configurations and switches on random and industrial instances.}
\tabcolsep=1pt
\begin{center}
\begin{scriptsize}
\begin{tabular}{l |rlrl |rlrl |rlrl |rlrl }
\multirow{2}{*}{instance} & \multicolumn{4}{|c}{ \lba } & \multicolumn{4}{|c}{ \bfr } & \multicolumn{4}{|c}{ \fullnoswitch } & \multicolumn{4}{|c}{ \fullnoconf } \\
& \multicolumn{2}{|c}{\NCONF} & \multicolumn{2}{c}{\NSWITCH}  & \multicolumn{2}{|c}{\NCONF} & \multicolumn{2}{c}{\NSWITCH}  & \multicolumn{2}{|c}{\NCONF} & \multicolumn{2}{c}{\NSWITCH}  & \multicolumn{2}{|c}{\NCONF} & \multicolumn{2}{c}{\NSWITCH}  \\
\hline
\hline
\instance{ 030$\cdot$04 } & \cellcolor{TealBlue!30}{6.4} & \cellcolor{TealBlue!30}{(5)} & \cellcolor{TealBlue!30}{1.6} & \cellcolor{TealBlue!30}{(5)} & \cellcolor{TealBlue!30}{6.4} & \cellcolor{TealBlue!30}{(5)} & \cellcolor{TealBlue!30}{1.6} & \cellcolor{TealBlue!30}{(5)} & \cellcolor{TealBlue!30}{6.4} & \cellcolor{TealBlue!30}{(5)} & 1.6 & (4) & \cellcolor{TealBlue!30}{6.4} & \cellcolor{TealBlue!30}{(5)} & 1.6 & (3) \\
\instance{ 030$\cdot$06 } & \cellcolor{TealBlue!30}{3.4} & \cellcolor{TealBlue!30}{(5)} & \cellcolor{TealBlue!30}{0.8} & \cellcolor{TealBlue!30}{(5)} & \cellcolor{TealBlue!30}{3.4} & \cellcolor{TealBlue!30}{(5)} & \cellcolor{TealBlue!30}{0.8} & \cellcolor{TealBlue!30}{(5)} & \cellcolor{TealBlue!30}{3.4} & \cellcolor{TealBlue!30}{(5)} & 0.8 & (2) & \cellcolor{TealBlue!30}{3.4} & \cellcolor{TealBlue!30}{(5)} & 0.8 & (2) \\
\instance{ 050$\cdot$04 } & 5.2 & (3) & 2.2 & (3) & 5.2 & (4) & \cellcolor{TealBlue!30}{2.0} & \cellcolor{TealBlue!30}{(4)} & \cellcolor{TealBlue!30}{5.2} & \cellcolor{TealBlue!30}{(5)} & 2.0 & (3) & \cellcolor{TealBlue!30}{5.2} & \cellcolor{TealBlue!30}{(5)} & \cellcolor{TealBlue!30}{2.0} & \cellcolor{TealBlue!30}{(4)} \\
\instance{ 050$\cdot$06 } & 3.6 & (4) & 0.6 & (4) & \cellcolor{TealBlue!30}{3.6} & \cellcolor{TealBlue!30}{(5)} & \cellcolor{TealBlue!30}{0.6} & \cellcolor{TealBlue!30}{(5)} & \cellcolor{TealBlue!30}{3.6} & \cellcolor{TealBlue!30}{(5)} & 0.6 & (2) & \cellcolor{TealBlue!30}{3.6} & \cellcolor{TealBlue!30}{(5)} & 0.6 & (2) \\
\instance{ 080$\cdot$04 } & 9.2 & (1) & 16.6 & (1) & 9.0 &  & 16.8 &  & \cellcolor{TealBlue!30}{9.0} & \cellcolor{TealBlue!30}{(3)} & 14.4 & (1) & \cellcolor{TealBlue!30}{9.0} & \cellcolor{TealBlue!30}{(3)} & \cellcolor{TealBlue!30}{13.4} & \cellcolor{TealBlue!30}{(1)} \\
\instance{ 080$\cdot$06 } & 5.4 & (3) & 5.2 & (3) & 5.4 & (3) & 4.4 & (3) & 5.4 & (4) & 4.6 & (3) & \cellcolor{TealBlue!30}{5.4} & \cellcolor{TealBlue!30}{(5)} & \cellcolor{TealBlue!30}{3.8} & \cellcolor{TealBlue!30}{(3)} \\
\instance{ 100$\cdot$04 } & 13.8 &  & 37.0 &  & 13.8 &  & 37.0 &  & \cellcolor{TealBlue!30}{12.8} &  & 30.8 &  & \cellcolor{TealBlue!30}{12.8} &  & \cellcolor{TealBlue!30}{30.4} &  \\
\instance{ 100$\cdot$06 } & 4.0 &  & 4.6 &  & 4.0 & (1) & 4.4 & (1) & \cellcolor{TealBlue!30}{4.0} & \cellcolor{TealBlue!30}{(5)} & \cellcolor{TealBlue!30}{2.8} & \cellcolor{TealBlue!30}{(5)} & \cellcolor{TealBlue!30}{4.0} & \cellcolor{TealBlue!30}{(5)} & \cellcolor{TealBlue!30}{2.8} & \cellcolor{TealBlue!30}{(5)} \\
\instance{ 200$\cdot$04 } & 11.0 &  & 57.0 &  & 11.0 &  & 56.6 &  & \cellcolor{TealBlue!30}{10.0} &  & \cellcolor{TealBlue!30}{47.2} &  & \cellcolor{TealBlue!30}{10.0} &  & 47.8 &  \\
\instance{ 200$\cdot$06 } & \cellcolor{TealBlue!30}{5.0} &  & 18.8 &  & \cellcolor{TealBlue!30}{5.0} &  & 17.6 &  & \cellcolor{TealBlue!30}{5.0} &  & \cellcolor{TealBlue!30}{14.4} &  & \cellcolor{TealBlue!30}{5.0} &  & 14.8 &  \\
\instance{ 300$\cdot$04 } & 11.0 &  & 98.8 &  & 11.0 &  & 97.4 &  & \cellcolor{TealBlue!30}{10.0} &  & \cellcolor{TealBlue!30}{81.8} &  & \cellcolor{TealBlue!30}{10.0} &  & 85.8 &  \\
\instance{ 300$\cdot$06 } & 5.0 &  & 28.2 &  & 5.0 &  & 29.8 &  & \cellcolor{TealBlue!30}{4.0} &  & 22.0 &  & \cellcolor{TealBlue!30}{4.0} &  & \cellcolor{TealBlue!30}{21.8} &  \\
\hline
\instance{ A cold } & \cellcolor{TealBlue!30}{6.0} &  & 10.2 &  & \cellcolor{TealBlue!30}{6.0} &  & 7.6 &  & \cellcolor{TealBlue!30}{6.0} &  & 12.0 &  & \cellcolor{TealBlue!30}{6.0} &  & \cellcolor{TealBlue!30}{7.2} &  \\
\instance{ A hot } & 4.0 &  & 3.0 &  & 4.0 &  & 2.8 &  & \cellcolor{TealBlue!30}{4.0} & \cellcolor{TealBlue!30}{(5)} & 3.2 &  & \cellcolor{TealBlue!30}{4.0} & \cellcolor{TealBlue!30}{(5)} & \cellcolor{TealBlue!30}{2.0} &  \\
\instance{ B cold } & \cellcolor{TealBlue!30}{9.0} &  & 64.8 &  & 9.2 &  & 65.4 &  & \cellcolor{TealBlue!30}{9.0} &  & \cellcolor{TealBlue!30}{50.6} &  & \cellcolor{TealBlue!30}{9.0} &  & 52.2 &  \\
\instance{ B hot } & 4.0 &  & 3.8 &  & 4.0 &  & 4.0 &  & \cellcolor{TealBlue!30}{4.0} & \cellcolor{TealBlue!30}{(5)} & \cellcolor{TealBlue!30}{1.8} &  & 4.0 &  & 2.4 &  \\
\instance{ C cold } & \cellcolor{TealBlue!30}{8.0} &  & 55.4 &  & \cellcolor{TealBlue!30}{8.0} &  & 55.2 &  & \cellcolor{TealBlue!30}{8.0} &  & \cellcolor{TealBlue!30}{49.2} &  & \cellcolor{TealBlue!30}{8.0} &  & 50.0 &  \\
\instance{ C hot } & 3.0 &  & 5.8 &  & 3.0 &  & 5.2 &  & \cellcolor{TealBlue!30}{3.0} & \cellcolor{TealBlue!30}{(5)} & 0.8 &  & \cellcolor{TealBlue!30}{3.0} & \cellcolor{TealBlue!30}{(5)} & \cellcolor{TealBlue!30}{0.4} & \cellcolor{TealBlue!30}{(3)} \\
\hline
\hline
\end{tabular}
\end{scriptsize}
\end{center}
\end{table}

%% file: tableCM.tex
\begin{table}[t]
\caption{\label{tab:CM}Comparison with current methods. Number of configurations and switches on industrial instances.}
\tabcolsep=3pt
\begin{center}
	\begin{scriptsize}
%\begin{small}
\begin{tabular}{l  |rr |rr |rlrl }
\multirow{2}{*}{instance} & \multicolumn{2}{|c}{ \airbus } & \multicolumn{2}{|c}{ \airbustsp } & \multicolumn{4}{|c}{ \algo{Choco} (\full) }  \\
& \multicolumn{1}{|c}{\NCONF} & \multicolumn{1}{c}{\NSWITCH}  & \multicolumn{1}{|c}{\NCONF} & \multicolumn{1}{c}{\NSWITCH}  & \multicolumn{2}{|c}{\NCONF} & \multicolumn{2}{c}{\NSWITCH}  \\
\hline
\hline
\instance{ A cold } 
& 6 & 43 & 6 & 9 
 & {6} &  & \cellcolor{TealBlue!30}{7.0} &  \\
\instance{ A hot }
& 4 & 5 & 4 & 3
& \cellcolor{TealBlue!30}{4} & \cellcolor{TealBlue!30}{(5)} & \cellcolor{TealBlue!30}{2.0} &  \\
\instance{ B cold }
& 10 & 108 & 10 & 70
& \cellcolor{TealBlue!30}{9} &  & \cellcolor{TealBlue!30}{50.2} &  \\
\instance{ B hot } 
& 5 & 30 & 5 & 1
& \cellcolor{TealBlue!30}{4} & \cellcolor{TealBlue!30}{(5)} & \cellcolor{TealBlue!30}{0.8} & \cellcolor{TealBlue!30}{(2)} \\
\instance{ C cold } 
& 8 & 91 & 8 & \cellcolor{TealBlue!30}{41}
& {8} &  & 50.6 &  \\
\instance{ C hot } 
& 3 & 14 & 3 & 9 
& \cellcolor{TealBlue!30}{3} & \cellcolor{TealBlue!30}{(5)} & \cellcolor{TealBlue!30}{0.0} & \cellcolor{TealBlue!30}{(5)} \\
\hline
\hline
\end{tabular}
%\end{small}
\end{scriptsize}
\end{center}
\end{table}